\documentclass[twoside]{article}

\usepackage[accepted]{aistats2025}
\usepackage[round]{natbib}

\usepackage{enumitem}

\usepackage[utf8]{inputenc} % allow utf-8 input
\usepackage[T1]{fontenc}    % use 8-bit T1 fonts
\usepackage{lmodern}
\usepackage{hyperref}       % hyperlinks
\usepackage{url}            % simple URL typesetting
\usepackage{booktabs}       % professional-quality tables
\usepackage{amsfonts}       % blackboard math symbols
\usepackage{nicefrac}       % compact symbols for 1/2, etc.
\usepackage{microtype}      % microtypography
\usepackage{xcolor}         % colors
\usepackage{wrapfig}

\hypersetup{
    colorlinks,
    linkcolor={red!50!black},
    citecolor={blue!60!black},
    urlcolor={blue!80!black}
}

\usepackage{nicefrac}
\usepackage{tikz}
\usetikzlibrary{arrows,shapes}
\usepackage[rectangularnodes]{influence-diagrams}

\usepackage[capitalize,noabbrev]{cleveref}

\usepackage{YK}

\RequirePackage{algorithm}
\RequirePackage{algorithmic}

\usepackage{aistats2025}
\begin{document}


\runningauthor{Bracale, Maity, Sun, Banerjee}

\twocolumn[

\aistatstitle{Learning the Distribution Map in Reverse Causal Performative Prediction}

\aistatsauthor{ Daniele Bracale$^*$ \And Subha Maity$^*$ \And  Yuekai Sun \And Moulinath Banerjee }

\aistatsaddress{ dbracale@umich.edu\\
Department of Statistics\\
University of Michigan \And smaity@uwaterloo.ca\\Department of Statistics \\ \& Actuarial Science \\University of Waterloo 
\And yuekai@umich.edu\\Department of Statistics\\University of Michigan \And moulib@umich.edu\\Department of Statistics\\University of Michigan } ]

\renewcommand{\thefootnote}{} % Remove the counter
\footnotetext{* Equal contribution.}
\renewcommand{\thefootnote}{\arabic{footnote}} % Restore the counter

\begin{abstract}

 In numerous predictive scenarios, the predictive model affects the sampling distribution; for example, job applicants often meticulously craft their resumes to navigate through a screening system. Such shifts in distribution are particularly prevalent in social computing, yet, the strategies to learn these shifts from data remain remarkably limited. Inspired by a microeconomic model that adeptly characterizes agents' behavior within labor markets, we introduce a novel approach to learning the distribution shift. Our method is predicated on a \emph{reverse causal model}, wherein the predictive model instigates a distribution shift exclusively through a finite set of agents' actions. Within this framework, we employ a microfoundation model for the agents' actions and develop a statistically justified methodology to learn the distribution shift map, which we demonstrate to effectively minimize the performative prediction risk. 
\end{abstract}

\section{Introduction}

Predictive models frequently guide decisions in strategic settings, but their environmental impact is often ignored. For example, a lender assessing loan default risk to set interest rates might inadvertently increase default risk for riskier applicants by charging higher rates, creating a self-fulfilling prophecy. In addition, borrowers might manipulate their profiles to appear less risky and secure lower rates. These influences are widespread: crime predictions shape police resource allocation, affecting crime rates, while recommendations alter user preferences and consumption. Goodhart's law \citep{strathern1997improving}, summarized as ``when a measure becomes a target, it ceases to be a good measure," also hints how predictive models affect their environments through performative \emph{agents}.

% Predictive models routinely aid decisions in strategic environments, yet their impact on the environment is often overlooked. Take, for instance, a lender wishes to assess default risk from loan applications to make informed lending decisions: loans are offered at higher interest rates to those deemed riskier. Here are two ways in which such assessment impacts the environment:
% \begin{enumerate}[left=1em]
% \item increased interest rates escalate the debt burden, thereby heightening the default risk in a self-fulfilling prophecy,
% \item borrowers seeking lower interest rates may ``game'' the default risk assessment model by doctoring their profiles to appear less risky.
% \end{enumerate}
% Once recognized, the influence of predictive models is evident across various applications: crime hotspot predictions inform policing resource allocations, which in turn affects criminal activity, and recommendations bias user preferences, which in turn affects consumption. Goodhart's law \citep{strathern1997improving}, often glibly described as ``when a measure becomes a target, it ceases to be a good measure'', acknowledges the ubiquitous effects of predictive models on their environment. It also hints at the common mechanism by which predictive models influence their environments through performative \emph{agents}.

From a probabilistic perspective, performativity manifests itself as a shift in distribution; agents adapt to prediction models by altering their attributes over time. Returning to the default risk assessment example, the self-fulfilling prophecy leads to label shift \citep{lipton2018Detecting}, while borrowers manipulating their profiles cause a class conditional shift \citep{maity2022Understanding}.  To mitigate these changes, while practitioners often retrain models \citep{cormier2016Launch}, a more proactive measure is to incorporate environmental impacts during model training, known as \textit{strategic/performative prediction} \citep{hardt2015Strategic,perdomo2020Performative}, which adjusts risk minimization to account for distribution shifts. This method, while effective, requires understanding these shifts -- a topic not well-covered in the literature. \emph{Our paper proposes a framework to estimate these distribution shifts}, which makes training of prediction tools for performative environments more accessible to practitioners.

\begin{figure}[h]
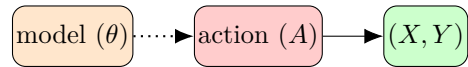

    \centering
    \vspace{-0.1in}
    \begin{influence-diagram}
      \node (X) [fill=orange!20!white] {model ($\theta$)};
      \node (A) [right =0.8cm of X,fill=red!20!white] {action ($A$)};
      \node (Z) [right =0.8cm of A, fill=green!20!white] {$(X, Y)$};
      \edge[information] {X} {A};
      \edge {A} {Z};
    \end{influence-diagram}
    \vspace{-0.1in}
    \caption{Reverse causal performative prediction}
    \label{fig:reverse_causal}
\end{figure}

In this paper, we reformulate the setup of {\it reverse causal performative prediction} as introduced in \cite{somerstep2024learning}, which can be roughly stated as ``the agents doctor their attributes by taking appropriate actions,  and
the distributions of modified attributes depend on the predictive model only through these actions''. To formalize the assumption, we denote the \emph{ex-post} distribution of attribute $(X, Y, A) \in \cX \times \cY \times \cA$ after deploying a model $\theta \in \Theta$ as $\cD(\theta)$, where $X$ denotes the covariates, $Y$ is the response and $A$ is the action. The parameter space $\Theta$ is of appropriate dimension, depending on the model used by the learner (\eg\  $\dim(\Theta) = \dim(\cX)+1$ for linear models and $\dim(\Theta) \gg \dim(\cX)$ for neural networks). Our reformulation states that $\cD(\theta)$ depends only through the distribution of agents' action, \ie\ having specified an action $a \in \cA$  the $\cD(X , Y \mid A =a ;\theta)$ is invariant of $\theta$ and therefore, denoting the invariant distribution as $\cD(X, Y \mid A)$ the $\cD(\theta)$ is decomposed as 
\begin{equation}\label{assumption:main}
    \textstyle\cD(\theta) =  \cD_A(\theta) \times \cD(X, Y \mid A) \,.
\end{equation} In the second paragraph of our conclusion \cref{sec:conclusion} we discuss this reverse causal model into the broader context of performative prediction, but for now we move on toward  our goal of learning $\cD(\theta)$ within this reverse causal model. To learn $\cD(\theta)$, in our paper, we assume that the learner has access to adequate $\iid$  samples from each of the $\cD(X, Y \mid A)$, or better yet, we assume that $\cD(X, Y \mid A)$ is known. Several instances where $\cD(\cdot \mid A)$ is known can be found in \citet[§§ 2.1, 3.1, 3.2, 4.1, 4.3, 6.6.1, 6.6.2]{fang2011theories} where they analyze models in which $Z$ represents a score and $A$ denotes group membership (e.g., \emph{qualified} vs. \emph{unqualified}) \citep[§§ 3.2, 4.1, 4.3]{fang2011theories}. Then, the only part that remains to be learned is $\cD_A(\theta)$, \ie\ the distribution map for the actions of the agents. A  proposal for learning this $\cD_A(\theta)$ follows.  

To estimate $\cD_A(\theta)$, we consider a microfoundation framework whose motivation comes from the Caote-Loury labor market model \citep{coate1993Will} (see \cref{ex:coate-loury}). The best course of action of an agent is formalized by the \emph{cost-adjusted utility maximization problem}: 
\begin{equation} \label{eq:microfoundation-random}
  \textstyle  A \triangleq \argmax_{a \in \cA} \sB_a(\theta) - C_a\,,
\end{equation} where  $\sB_a(\theta)$ is the benefit of taking an action $a$ under the model $\theta \in \Theta$, and  $C_a$ is the cost of the same $a$.  Typically, the $\sB_a(\theta)$ is known to the learner. Furthermore, we frame $C_a$ as a random quantity, which encodes an agent's own preference for the action $a$. For example, while preparing for a test, a candidate from a wealthier background would be more willing to enroll in a test-prep course. Given that $\sB_a(\theta)$ is known, in this microfoundation setting, we learn $\cD_A(\theta)$ through learning the distribution of $\{C_a; a \in \cA\}$.  To keep things simple, in this paper we assume that \emph{agent's action space $\cA$ is finite}. While it may be possible to learn this distribution for an arbitrary $\cA$, this will involve considerably more technical challenges and beyond the scope of this paper; further related discussion is deferred to the concluding \cref{sec:conclusion}.
% \Todo{emphasize the benefit of this microfoundation? that once learned, the cost distribution is completely independent of $f$ which may simplify training of model?}

At this point, it may seem that we have restricted our reverse causal setting to merely a microfoundation setting described in \eqref{eq:microfoundation-random}, where the agents are necessarily strategic. But we argue in \cref{sec:microfoundation} (\cf\ \cref{lemma:cost-existance}) that by incorporating a random cost into our microfoundation we make it completely general in the reverse causal setting, as, for a suitable choice of benefit $\bB(\theta) \triangleq \{\sB_a(\theta): a\in \cA\}$ and random cost $\bC \triangleq \{C_a: a \in \cA\}$, any $\cD_A(\theta)$ can be expressed by our microfoundation setup. This is not possible in a usual strategic prediction setting \citep{hardt2015Strategic}, where the deterministic cost can only allow a strategic agent. In contrast, our setting also allows for \emph{non-strategic agents}. By \emph{non-strategic} we mean that two agents with the same \emph{ex-ante} $z$ do not necessarily have the same \emph{ex-post} $z'$ if the cost is random. We defer further related discussion to \cref{sec:microfoundation}. 

We end the introductory section with a description of our contribution, followed by a brief survey of the related literature. 

\begin{enumerate}[left=1em]
\item In \cref{sec:reverse_causal} we introduce the microfoundations in reverse causal models for learning the distribution map, and we show that a microfoundation framework can recover any distribution shift in a reverse causal model by randomizing the cost.

\item In \cref{sec:learning-cost} we provide a method to estimate the distribution map non-parametrically, for a general number of actions. 
\item In \cref{sec:optimal_design} we provide a sequential \emph{optimal design} algorithm based on the doubling trick (\ie, every episode is twice the previous one) with $2$ actions. We prove a convergence result of the \textit{relative efficiency} of the estimated design density.
\item In \cref{sec:Regret} we propose an algorithm based on the doubling trick to estimate the distribution map $\cD(\theta)$ with a regret bound of $\widetilde{\mathcal{O}}(M^{\frac{1}{1+\eta}})$, where $\eta$ is the uniform convergence rate of the distribution map in the exploration phase and $\widetilde{\mathcal{O}}$ excludes factors $\log(M)$.
\item We consider the Coate-Loury labor market model \cite{coate1993Will} as a baseline, for which we show convergence results of the relative efficiency in both theoretical and numerical experiments.
\end{enumerate}

\subsection{Related work} 
% \todo{Rewrite it since this is borrowed from the continuous case}
 The performative prediction, initially introduced by \citet{perdomo2020Performative}, encompasses various aspects, with a notable instance being strategic classification \cite{hardt2016Strategic,shavit2020Causal,levanon2022Generalized}. In the context of strategic prediction, agents are presumed to be utility-maximizing, leading to an implicit determination of the distribution map by the agents' utilities. A growing body of literature focuses on developing algorithms for computing performative optimal policies \cite{izzo2021How,levanon2021Strategic,miller2021Outside}. In a related line of research, performativity is addressed in online settings with the aim of devising algorithms that effectively minimize Stackelberg regret \cite{dong2018Strategic,chen2020Learning,jagadeesan2022Regret}.

 % Strategic classification has been expanded on in several directions: including models with causality \cite{mendler-dunner2022Anticipating, horowitz2023Causal, somerstep2023Learning, harris2022Strategic}, models with opaque agent behavior \cite{ghalme2021Strategica}, the effects of strategic behavior on graph neural networks \cite{eilat2023strategic}, models with reversed order of play \cite{zrnic2022Who}, PAC learning for strategic classification \cite{sundaram2021paclearning},  strategic ordinary least squares \cite{shavit2020Causal} and combinations of these \citep{levanon2022Generalized}. Strategic behavior with an element of competition between agents has been used to model the behavior of content creators in socio-technical systems \citep{hron2023Modeling, jagadeesan2023SupplySidea}. The study of strategic behavior is not isolated to computer science; strategic behavior is the key ingredient to many micro-economic models of labor markets \citep{coate1993Will, moro2003Affirmative, moro2004general, fang2011Chapter}. 

% A related line of research addresses performativity in online settings, aiming to devise algorithms that minimize Stackelberg regret effectively \cite{dong2018Strategic,chen2020Learning,jagadeesan2022Regret}. 

A key challenge that impedes the broader adoption of performative prediction lies in practitioners often not knowing the distribution map. In general, there are two ways to bypass this problem. The first treats the performative prediction problem \eqref{eq:performative-prediction} as a zeroth-order (\ie\ derivative free) optimization problem, relying on samples from $\cD(\theta_t)$ ($\theta_t$ being optimization algorithm iterates) to solve it without explicit knowledge of the distribution map $\cD(\theta)$ \cite{izzo2021How,miller2021Outside}. However, as the authors \cite{miller2021Outside} state, this method often leads to slow convergence and curse of dimensionality.
The second approach deals with unknown agent responses by imposing microfoundation models \cite{jang2022Sequential}. This approach benefits from fast convergence rates, leveraging white-box access to the distribution map \cite{hardt2016Strategic,levanon2021Strategic,cutler2021Stochastic}, but the distribution map is often unknown in practice, and a misspecification in this map can hinder the finding of the correct solutions \cite{lin2023plug}. Our proposed approach complements the second approach: by learning in a non-parametric way the microfoundations model from agent response data, we address issues related to misspecified distribution map, while benefiting from faster optimization algorithms by utilizing the white-box access to the distribution map. To the best of our knowledge, current results for estimating the distribution map assume a parametric model on $\cD(\theta)$, for example, \cite{miller2021Outside} assumes a location-scale distribution map generated by the linear transformation $z_{\theta} = z_0 + \mu \theta$. %Our prior work on anti-causal strategic classification \cite{somerstep2023Learning} exemplifies this strategy. Unfortunately, misspecified agent response models often arise \cite{lin2023Plugina}, leading to inaccurate distribution maps and hindering correct performative prediction problem solutions.

\section{Reverse causal performative prediction}
\label{sec:reverse_causal}

In a standard strategic prediction setting, agents update their covariates $X$, but do not update their responses $Y$. \citet{horowitz2023causal} considered a more general setting in which agents modify both $X$ and $Y$ using a causal/structural model. Here, agents still change their covariates, but such changes can propagate through the causal model and (indirectly) lead to changes in their responses. In their reverse causal model, \citet{somerstep2024learning} assumes that the agents update their $Y \to Y'$, which leads to a change in $X \sim \Phi(\cdot \mid Y')$.  We further generalize this by assuming that the agents can modify both $Z= (X, Y)$ as a response to $\theta$, but only through their action $A$, leading to the structural model in \cref{fig:reverse_causal}. Thus, the \emph{ex-post} distribution $\cD(\theta)$ depends on $\theta$ only through the marginal distribution of $A$, \ie\ through $\cD_A(\theta)$. Since $\cD(\theta)$ can be decomposed as $\cD(\theta) = \cD_A(\theta) \times \cD(Z \mid A;\theta)$, the reverse causal setting is equivalently stated in the following assumption. 
\begin{assumption}[Reverse causal model]\label{assumption:condition_X_given_A}
The $\cD(Z\mid A;\theta)$ are not affected by the prediction model $\theta$. Moving forward, we drop this dependence on $\theta$  and simply denote this by $\cD(Z \mid A)$. 
\end{assumption} 
Our reverse causal model simplifies to the one in \citet[Section 2]{somerstep2024learning} if we let $A = Y$.

A prominent example of reverse causal model is the Coate-Loury (CL) labor market model \citep{coate1993Will}. We describe a variation of this model below, which will serve as a running example throughout our paper.
\begin{example}[Coate-Loury labor market model \citep{coate1993Will}]
\label{ex:coate-loury}
 Employers seek to hire skilled workers, but while making hiring decisions, are uncertain whether the workers are skilled ($Y =1$) or unskilled ($Y =0$). Instead, they observe a score $X \in \cX=[0,1]$, which is an (imperfect) skill assessment of workers and hire those whose scores exceed a threshold $\theta \in [0,1]$, \ie\ hires if $X > \theta$. Workers respond strategically to the hiring policy by making skill investments:
\begin{equation}
A\triangleq\argmax_{a\in\{0,1\}}\omega \bar{F}(\theta\mid a) - C\cdot a,
\label{eq:CL-worker-strategy}
\end{equation}
where $\omega  > 0$ is the wage paid to hired workers and $\bar{F}(\theta\mid a)=\bbP(X>\theta|Y=a)$ is the survival function of $X$, interpreted as the probability of hiring a (skilled and unskilled) worker,  and finally, $C > 0$ is a random cost of worker-specific investment for acquiring skill $Y = 1$. Here, the action $A$ is identical to the response $Y$: whether to be skilled ($A = Y = 1$) by making an investment or not ($A = Y = 0$), and thus we absorb $Y$ within $A$ while describing the distribution shift. As the distribution of $X \mid A$ remains unaffected by the threshold $\theta$, the only component of the distribution of $(X,A)$ that changes is the marginal distribution of $A$, leading to the distribution shift map
\[ \textstyle
\cD(X, A;\theta) = \cD_A(\theta)\times \cD(X \mid A), ~~ \cD_A(\theta) \stackrel{\textrm{d}}{=} \operatorname{Ber}(\pi(\theta))
\]
where $\stackrel{\textrm{d}}{=}$ denotes `equally distributed' and $\pi(\theta)$ is the proportion of skilled workers, determined  by the workers whose (expected) wage incentive $I(\theta)\in [0,\omega]$ as below: 
\begin{align}
\pi(\theta)\triangleq \Pr_{\theta}\big \{A=1\big \} = \Pr\big \{C \leq I(\theta)\big \},\\
I(\theta)\triangleq \omega\big\{\bar{F}(\theta\mid 1) - \bar{F}(\theta\mid 0)\big\}\,.
\label{eq:Coate-Loury-dist-map-param}
\end{align}

% distribution map (defined in \eqref{eq:performative-prediction}) for this example is
% $$
% \cD_A(\theta) = \operatorname{Bernoulli}(\pi(\theta))
% $$
% where $\pi(\theta)$ is the proportion of workers who become skilled. Estimating $\cD_A$ boils down to estimating $\pi$. Workers whose (expected) wage incentive \[I(\theta)\triangleq w[\bar{F}(\theta\mid 1) - \bar{F}(\theta\mid 0)]\]
% exceeds their skill investment cost become skilled, and the aggregate response of the workers is 
% \begin{equation}
% \pi(\theta)\triangleq \Pr[C \leq I(\theta)].
% \label{eq:Coate-Loury-dist-map-param}
% \end{equation}
% Then, estimating $\cD_A(\theta)$ reduces to estimating the cdf of $C$.
\end{example}

\subsection{Learning prediction model in reverse causal model}
As a proactive approach to learning a model in a performative prediction setting, \citet{hardt2015Strategic,perdomo2020Performative} propose to minimize \emph{performative risk} by modifying the standard risk minimization problem to account for the shift in the underlying distribution:
\begin{equation}\label{eq:performative-prediction}
\textstyle \min_{\theta \in \Theta}\PR(\theta),\quad \textstyle \PR(\theta)\triangleq \Ex_{ \cD(\theta)} [\ell(\theta; Z)]\,,
\end{equation} where $\ell$ is an appropriate loss function. For our reverse causal model (\cref{assumption:condition_X_given_A}), the performative risk simplifies to 
\begin{equation}
\begin{aligned}
    \textstyle
\PR(\theta)
= \textstyle \sum\nolimits_{a \in \cA} \cD_{A=a}(\theta)\Ex\big[\ell(\theta;Z)\mid A = a\big]\,,
\end{aligned}
\label{eq:PR-discrete-actions}
\end{equation} where $\Ex[\ell(\theta;Z)\mid A]$ is efficiently estimated given adequate samples from the distribution of $Z \mid A = a$. To keep things simple, we assume that $\cD(Z \mid A = a)$ is known for the moment.

\begin{example}[Coate-Loury model] Returning to \cref{ex:coate-loury},  assume that the firm with hiring policy $\widehat Y_\theta(X) = \indicator\{X > \theta\}$ incurs a loss of $\delta_0 > 0$ (resp. gain of $\delta_1 > 0$) for hiring an unskilled (resp. skilled) worker. Expressing the gain as a negative loss, we can express the loss function as $\ell(\theta; Z) \triangleq \widehat Y_\theta(X) \{ - \delta_1 Y + \delta_0(1 - Y)\}$. In this case the performative risk $\textstyle\PR(\theta)$ is: 
\begin{align}
&\textstyle\sum_{a \in \{0,1\}} \cD_{A=a}(\theta)\Ex\big[\ell(\theta;Z)\mid A= a\big] \\
& = - \pi(\theta) \delta_1 \bar F(\theta \mid 1) + (1 - \pi(\theta)) \delta_0 \bar F(\theta \mid 0).\nonumber
\end{align}
    
\end{example}

\subsection{Microfoundation framework for learning the distribution shift} \label{sec:microfoundation}
In order to minimize the performative risk \eqref{eq:PR-discrete-actions}, we require $\cD(\theta)$. Since we assume that we know $\Ex[\ell(\theta; Z)\mid A]$, it remains to learn $\cD_A(\theta)$. For this purpose, we propose a microfoundation framework: an agent pays a random cost $C_a$ for their action $a$ and strategically chooses their best course of action by maximizing their cost-adjusted utility, as described in \eqref{eq:microfoundation-random}
\[
\textstyle  A \triangleq \argmax_{a \in \cA} \sB_a (\theta) - C_a
\] where, for an action $a$, $\sB_a(\theta)$ is the benefit realized under the model $\theta$, and $C_a$ is the random cost that encodes the agent's own preference for the action. We eliminate the edge cases where the above $\argmax$ has multiple solutions and an agent is in conflict between multiple actions. Such conflictions can be easily eliminated by assuming that the distributions of $C_a$'s are continuous, for which the $\argmax$ is unique with probability one. 

\begin{example}[Coate-Loury model] In our \cref{ex:coate-loury}, one may set $\sB_a( \theta) = \omega \bar F(\theta \mid A = a)$, $C_1 = C$ and $C_0 \equiv 0$. 
    
\end{example}

Under this microfoundation framework \eqref{eq:microfoundation-random} the $\cD_A(\theta)$ is expressed as: 
\begin{equation} \label{eq:D_f_a}
    \cD_{A=a}(\theta) = \Pr\big\{\sB_a( \theta)-C_a \ge \sB_{a'}( \theta)-C_{a'}, \forall a'\neq a\big\}\,.
\end{equation} Typically, the $\sB_a( \theta)$ is known to the learner, \eg\ in Coate Loury \cref{ex:coate-loury}, the firm typically has access to $\cD(X \mid A)$ from a historical dataset, and can easily calculate $\sB_a( \theta) = \bbP(X > \theta \mid A = a)$. It only remains to learn the distribution of $\bC \triangleq \{C_a; a\in \cA\}$ and, moving ahead, in \cref{sec:learning-cost} we propose an algorithm for learning this. But, before moving ahead, we make the following remark about the generality of the microfoundation framework in the broader context of learning $\cD(\theta)$ in a reverse causal model.

\paragraph{Generality of microfoundation framework:} How restrictive is our framework \eqref{eq:microfoundation-random}, given that agents are strategic? As we show here, by incorporating a random cost in our microfoundation, we made it completely general in the reverse causal model that may include \emph{non-strategic} agents as well. By \emph{non-strategic} we mean that two agents with the same \emph{ex-ante} $z$ do not necessarily have the same \emph{ex-post} $z'$ if the cost is random. Moreover, in the following lemma, we show that, for a suitably chosen benefit, any $\cD_A(\theta)$ can be expressed using the microfoundation framework. 
\begin{lemma} \label{lemma:cost-existance}
% We say that $\cD_\theta(A)$ can be expressed by our microfoundation model using a benefit function $\bB(\theta) \triangleq \{\sB_a( \theta); a\in \cA\}$ if there exists a random cost vector $\bC$ that combined with this benefit function induces the $\cD_\theta(A)$ through the microfoundation \eqref{eq:microfoundation-random}. Any $\cD_\theta(A)$ can be expressed by our microfoundation by specifying $\bB(\theta) \triangleq \{\sB_a( \theta); a\in \cA\}$ that satisfies the following:
% for every $a, a'\in \cA$ with $a \neq a'$ the $\cD_\theta(A = a)$ increases with respect to $\sB_a( \theta) - \sB_{a'}( \theta)$. 

For \emph{any} $\cD_A(\theta)$  and a benefit function $\bB(\theta) \triangleq \{\sB_a( \theta); a\in \cA\}$ satisfying the property that
for every $a, a'\in \cA$ with $a \neq a'$,  $\cD_{A=a}(\theta)$ can be written as a fixed  increasing function of $\sB_a( \theta) - \sB_{a'}( \theta)$, there exists a random cost vector $\bC$, that induces $\cD_A(\theta)$ through the  microfoundation equation \eqref{eq:microfoundation-random}. Furthermore, if $\{\cD_A(\theta) : \theta\in \Theta\} = \Delta^{\cA}$ then for a given $\bB(\theta)$ the distribution of such a $\bC$ that satisfies $C_0 = 0$ is unique.

   % The $\cD_\theta(A)$ can be expressed by our microfoundation model \eqref{eq:microfoundation-random} using a benefit function $\bB(\theta) \triangleq \{\sB_a( \theta); a\in \cA\}$, \ie\ there exists a random cost $\bC$ such that $\cD_\theta(A)$ can be expressed as \eqref{eq:D_f_a} if and only if for every $a, a'\in \cA$ with $a \neq a'$ the $\cD_\theta(A = a)$ increases with respect to $\sB_a( \theta) - \sB_{a'}( \theta)$. 
\end{lemma}
The complete proof is in \cref{appendix:proofs}, but here we explain the key idea in the simple case of $|\cA| = 2 $, \ie\ $\cA = \{0, 1\}$. We set $C_0 \equiv 0$. Since $\cD_{A=1}(\theta)$ increases with $\sB_1( \theta) - \sB_0( \theta) $, \ie\  $\cD_{A=1}(\theta) = G (\sB_1( \theta) - \sB_0( \theta))$ for an increasing function $G : \reals \to [0, 1]$, we extend $G$ to an appropriate c.d.f. $\tilde G$ over $\reals$, for which a random cost $C_1 \sim \tilde G $ leads to microfoundation \eqref{eq:microfoundation-random}, because
\begin{align}
&\Pr\big\{C_1 \le \sB_1( \theta) - \sB_0( \theta) \big\} = \tilde G (\sB_1( \theta) - \sB_0( \theta)) \nonumber\\
&= G (\sB_1( \theta) - \sB_0( \theta)) = \cD_{A=1}(\theta).
\end{align} If $\{\cD_1(\theta) : \theta\in \Theta\} = [0, 1]$ then the uniqueness of the distribution of $C_1$ is easily realized by noticing that the collection $\{ \sB_1( \theta) - \sB_0( \theta): \theta \in \Theta \}$ specifies all the quantiles of the distribution of $C_1$.

\paragraph{Misspecification of benefit function:} Although the learner must specify a $\bB(\theta)$ for the microfoundation, \cref{lemma:cost-existance} indicates that for any choice of $\bB(\theta)$ that meets the the lemma's condition, the microfoundation can accurately express $\cD_A(\theta)$. This allows a flexible choice of $\bB(\theta)$, as any misspecifications in the benefit are corrected within the estimated cost distribution, thereby ensuring the proper specification of $\cD_A(\theta)$.

\section{Learning the distribution of agents' action}
% via shape-constrained inference} 
\label{sec:learning-cost}

In this section, we propose a simple algorithm for estimating the $\cD_A(\theta)$, which is the remaining unknown part of the distribution shift. 
Recall the expression of $\cD_a(\theta) \triangleq \cD_{A = a}(\theta)$ in \eqref{eq:D_f_a}, which we first re-express in a compact vector inequality form with a matrix $\bL_a = \{\bL_a(u, v); u \in \cA - \{a\} , v \in \cA\}\in \reals ^{(\cA - \{a\}) \times \cA}$ defined as
\[
\bL_a(u, v) = \begin{cases}
    1 & \text{if}~ v = a, u \in \cA - \{ a\} \\
    -1 & \text{if} ~ u = v \neq a\\
    0 & \text{otherwise}
\end{cases}
\] With the $\bL_a$ the $\cD_a(\theta)$ can be compactly expressed as
\begin{equation}
    \begin{aligned}
        \cD_a(\theta) & = \textstyle  \Pr\big\{C_a-C_{a'} \le \sB_a( \theta)-\sB_{a'}(\theta), \forall a'\neq a\big\} \\
        &= \Pr\{\bL_a\bC \preceq \bL_a \bB(\theta)\} = \sF_{a}\big(\bL_a \bB(\theta)\big)
    \end{aligned} \label{eq:pi_a-multiple skills}
\end{equation}
where $\preceq$ is the element-wise inequality operator, $\bC\triangleq\{C_a; a \in \mathcal{A}\}$, $\bB(\theta)\triangleq\{\sB_a( \theta); a \in \mathcal{A}\}$, $\sF_{a}$ is the multivariate cumulative distribution function (cdf) of $\bL_a \bC$. 
Learning ${\cD}_A(\cdot)$ reduces to learning $\{\sF_{a}(\cdot);a\in\cA\}$. For this purpose, we assume that the learner deploys predictive models $\theta_1,\dots,\theta_M$ and observes the estimated proportions of the actions of agents $\widehat{\pi}(\theta_1),\dots,\widehat{\pi}(\theta_M)$, \ie, for each $m$ the $\widehat{\pi}(\theta_m) = \{\widehat{\pi}_{a}(\theta_m),a \in \mathcal{A}\}$ consists of the estimated proportion of agents $\widehat{\pi}_{a}(\theta_m)$ who chose to take the action $a$ as a response to the model $\theta_m$. In the following remark, we explain how a learner may estimate these proportions.

%% Prev version
% Let $\boldsymbol{\cD}_f(A) \triangleq \{\cD_f(a);a \in \mathcal{A}\}$, where $\cD_f(a)\triangleq \cD_f(A= a)$. According to \eqref{eq:D_f_a}
% \begin{align}
% \cD_f(a) = \Pr[L_a\bC \preceq L_a \bB(f)] = F_{a}(L_a \bB(f))
% \label{eq:pi_a-multiple skills}
% \end{align}\ignorespaces

% where $\preceq$ is the inequality operator applied element-wise, $\bC\triangleq\{C_a; a \in \mathcal{A}\}$, $\bB(f)\triangleq\{\sB_a( f); a \in \mathcal{A}\}$, $F_{a}$ is the multivariate cdf of $L_a \bC$ and $L_a$ is a $(|\mathcal{A}|-1)\times |\mathcal{A}|$ matrix defined as
% \[L_a =
% \begin{bmatrix}
%    -1 & \dots & 0 & 1 & 0 & \dots & 0 \\
%    \vdots & \ddots & \vdots & \vdots & \vdots & \vdots & \vdots \\
%    0 & \dots & -1 & 1 & 0 & \dots & 0 \\
%    0 & \dots & 0 & 1 & -1 & \dots & 0 \\
%    \vdots & \vdots & \vdots & \vdots & \vdots & \ddots & \vdots \\
%    0 & \dots & 0 & 1 & 0 & \dots & -1 \\
% \end{bmatrix}
% \]
% where the columns of $1$s is the $a$-th column. Then, learning $\boldsymbol{\cD}_f(A)$ reduces to learning $\{F_{a};a\in\cA\}$. To learn $F_{a}$, the learner deploys predictive models $f_1,\dots,f_M$ and observes the estimated agents' actions $\widehat{\pi}_{1},\dots,\widehat{\pi}_{M}$, where $\widehat{\pi}_{m} = [\widehat{\pi}_{a}(\theta_m)]_{a \in \mathcal{A}}$. 

\begin{remark}\label{remark:estimate-hatpi}
We explain two settings for estimating $\widehat\pi(\theta_m)$: with or without direct access to agents' actions. For the model $\theta_m$
\begin{enumerate}[label=(\arabic*),left=0em]
    \item If the learner has access to the actions $\{A_{m, i}; i \in [n_m]\}$ from $n_m$ $\iid$ draws from $\cD(\theta_m)$ then they can directly estimate the proportions as $\widehat \pi_{a}(\theta_m) = \nicefrac{1}{n_m} \sum\nolimits_{i = 1}^{n_m} \indicator\{ A_{m , i} = a\}$. 
    \item If the actions are not directly accessible to the learner, who can only access $\iid$ $\{Z_{m, i}; i \in [n_m]\}$ from $\cD(\theta_m)$, the moment matching approach from \citet{lipton2018Detecting} for estimating label shift can be easily adapted here. Since we assume that we know $\cD(Z\mid A; \theta_m)$, for a `suitable' choice of $h_m: \cZ \to \reals^{\cA}$, that has a good discriminating power over $\cD(Z \mid A; \theta_m)$, \ie\ smallest singular value of the following matrix is sufficiently larger than zero
    \[ \textstyle
    \Delta_m \triangleq \big[\Ex_{\cD(\theta_m)}[h_m(Z) \mid A = a]; a \in \cA\big] \in \reals^{\cA \times \cA}\,,
    \] the $\widehat\pi(\theta_m)$ can then be simply estimated as $\widehat \pi(\theta_m)= \Delta_m^{-1} \{\nicefrac{1}{n_m} \sum\nolimits_{i = 1}^{n_m} h_m(Z_{m, i})\}$. As explained in \citet[Theorem 3]{lipton2018Detecting}, the error in this estimator is inversely proportional to the smallest singular value of $\Delta_m$. 
\end{enumerate}

\end{remark}

% % Prev version
% \begin{remark}\label{remark:estimate-hatpi}
% The estimates $\widehat{\pi}_{m}$ are functions of the data. For example: suppose that after a $f_m$ has been deployed, the learner observes $n_m$ features $X_{m,1},\dots,X_{m,n_m}$ (where $n_m$ is the number of agents exposed to $f_m$) or the aggregate value $\sum\nolimits_{i=1}^{n_m}A_{m,i}$, where $A_{m,i}\in\{0,1\}^{|\cA|}$ are one-hot encodings of the action taken by the $i$-th agent exposed to $f_m$. Then $\widehat{\pi}_{m}$ can be estimated either from $\{X_{m,1},\dots,X_{m,n_m}\}$ (see Example \ref{ex:coate-loury_2}) or simply as $\widehat{\pi}_{m} = (\sum\nolimits_{i=1}^{n_m}A_{m,i})/n_m$. 
% \end{remark}

Having estimated the $\widehat\pi(\theta_m)$'s, the $\{\sF_a; a \in \cA\}$ can be learned by solving a least squares (LS) problem with an appropriate monotone shape constraint. Recall that $\sF_a$ is the cumulative distribution function (cdf) of $\bL_a \bC$ and thus a coordinate-wise monotone function. We incorporate this shape constraint in our LS problem:
% \begin{equation}
% \big\{ \widehat{\sF}_a; a \in \cA\big\} \in \underset{\{F_a;a\in \cA\}\in\cF}{\argmin}
% \mathscr{L} (\{F_a;a\in \cA\})\,
% \label{eq:coordinate-wise-monotone-regression}
% \end{equation}
% \todo[inline]{Why introduce this extra notation $\mathscr{L}$}
% where $\mathscr{L}(\{F_a;a\in \cA\})$ is equal to 
% \begin{equation*} 
% \left \{
% \begin{aligned}
%    &\textstyle  \sum_{a \in \cA,m \in [M]} n_m\big \{\sF_{a}\big(\bL_a \bB(\theta_m)\big) - \widehat{\pi}_{a}(\theta_m)\big \}^2,\\
%    & \text{subject to} ~ \textstyle \sum_{a\in\cA}\sF_a\big(\bL_a \bB(\theta_m)\big) = 1, ~\forall ~ m \in [M]\,,\\
%    % & \textstyle \text{and} ~ ~~~~~~~\sF_a\big(\bL_a \bB(\theta_m)\big) \ge 0 ~ \text{for all} ~ m \in [M], a \in \cA\,, 
% \end{aligned} \right .
% \end{equation*}
\begin{equation} \label{eq:coordinate-wise-monotone-regression}
    \begin{aligned}
       &\textstyle  \sum\limits_{a \in \cA,m \in [M]} n_m\big \{\sF_{a}\big(\bL_a \bB(\theta_m)\big) - \widehat{\pi}_{a}(\theta_m)\big \}^2,\\
   & \text{subject to} ~ \textstyle \sum\limits_{a\in\cA}\sF_a\big(\bL_a \bB(\theta_m)\big) = 1, ~\forall ~ m \in [M]\,,\\ 
    \end{aligned}
\end{equation}
where $\sF_a$ are multivariate cdfs. The loss corresponding to the estimate $\widehat\pi(\theta_m)$ is weighted by $n_m$ based on its estimation error: a $\widehat\pi(\theta_m)$ calculated from a larger sample of size $n_m$ is statistically more accurate and therefore is weighted with a weight of higher importance $n_m$. The LS problem in \eqref{eq:coordinate-wise-monotone-regression} is a coordinate-wise monotone regression problem \citep{chatterjee2018matrix,han2019Isotonic}; it is a convex quadratic program that can be efficiently solved. Armed with these $\widehat{\sF}_a$, the learner plugs them in to obtain (an estimate of) the distribution map $\widehat{{\cD}}_A(\theta)$, where for an $a \in \cA$ the $\widehat{\cD}_a(\theta) = \widehat{\sF}_a(\bL_a \bB(\theta))$ and any $\theta$. We summarize our approach in \cref{alg:discrete-actions}.

\begin{algorithm}
\caption{Learning discrete agent responses}\label{alg:discrete-actions}
\begin{algorithmic}[1]
\STATE Observe agent responses $\widehat \pi(\theta_1),\dots,\widehat \pi(\theta_M)\in\Delta^{|\cA|-1}$ to predictive models $\theta_1,\dots,\theta_M$
\STATE Estimate $\{ \sF_a$, $a\in\cA\}$ as $\{ \widehat{\sF}_a; a \in \cA\}$ with (coordinate-wise) monotone regression \eqref{eq:coordinate-wise-monotone-regression}
\STATE Estimate distribution map $\cD_{A}(\cdot)$ as $\widehat{\cD}_a(\cdot)\gets \widehat{\sF}_a(\bL_a \bB(\cdot))$, $a\in\cA$
\end{algorithmic}
\end{algorithm}

% In \cref{sec:MLE}) we propose an alternative approach for estimating $\cD_A(\cdot)$ that replaces \eqref{eq:coordinate-wise-monotone-regression} with an unconstrained multivariate isotonic regression.

\subsection{Example. Univariate cost: CL model.}
\label{ex:coate-loury_2}

Consider the CL model in \cref{ex:coate-loury}, where let $\{\widehat\pi(\theta_m): m \in [M]\}\subset [0,1]$ be the estimated proportions of qualified workers (\ie\ $\pi(\theta_m)\triangleq \cD_{A = 1}(\theta_m)$) after they have been exposed to the thresholds $\{\theta_m: m \in [M]\} \in \Theta =[0,1]$; see \cref{remark:estimate-hatpi} for their estimation. 
At this point, we set the cost of $A = 0$ (the worker remains unskilled) to zero and only estimate the random cost of becoming skilled ($A = 1$). Denoting this cost by $C$, the optimization in \eqref{eq:coordinate-wise-monotone-regression} simplifies to just estimating $F_C$ (c.d.f. of $C$) as 
\begin{equation} \label{eq:cdf-CL-model}
   \textstyle\widehat{F}_C \leftarrow \arg \min _{F \in \cF} \sum\nolimits_{m=1}^M n_m\{F(b_m)-\widehat\pi(\theta_m)\}^2, 
\end{equation} where $b_m = I(\theta_m)$ and $\cF$ is a collection of non-negative and monotonically increasing functions.

Next, we provide a uniform convergence result for the estimated c.d.f. in \eqref{eq:cdf-CL-model} assuming that the learner observes the $A_{m, i}$'s. While it is possible to establish a similar result for the second estimator proposed in \cref{remark:estimate-hatpi} which does not require these $A_{m, i}$'s, we leave it for future research interest as we believe this will be technically more challenging and verify the quality of estimated c.d.f. obtained from such $\widehat\pi(\theta_m)$ in a synthetic experiment. To establish uniform convergence, we assume that the true $F_C$ is smooth and the design points $b_m$ in \eqref{eq:cdf-CL-model} are  ``regular'', which are formalized below.
\begin{assumption}
\begin{enumerate}[label=(A\arabic*),left=0em]
    \item \label{assumption1} For an $0 \le \alpha< 1$ the true $F_C$ is $\alpha$-H\"older smooth, \ie\ there exists a  $K>0$ such that $\left|F_C(b)-F_C(b')\right| \leq K|b-b'|^\alpha$ for any $b, b'\in [0,1]$.
    \item \label{assumption2}  The design points $b_m$ are either equidistant over the range of $\{I(\theta): \theta \in \Theta\}=[0,\omega]$ ($\omega$ is defined in \cref{ex:coate-loury}) or drawn as $\iid$ from a density that is bounded away from zero.
\end{enumerate}
\end{assumption}

Now we establish a uniform convergence result in the following proposition, which follows directly from \cite[Theorem 3.3]{mosching2020monotone}.

\begin{proposition}\label{thm:uniform_convergence}
Suppose that $\widehat \pi(\theta_m)= \nicefrac{1}{n_m}\sum\nolimits_{i=1}^{n_{m}} A_{m,i}$ and define $n \triangleq \sum\nolimits_m n_m$.  Under assumptions \ref{assumption1} and \ref{assumption2} the  $\widehat F_C$ defined in \eqref{eq:cdf-CL-model} has the following rate of convergence:
$$ \textstyle
\sup_{b \in [\delta_n,w-\delta_n]}|\widehat{F}_C(b)-F_C(b)| = \mathcal{O}_P  \big \{\big(\nicefrac{\log n}{n} \big)^{\nicefrac{\alpha}{(2\alpha + 1)}}  \big \} 
$$
where, the $\delta_n = \kappa \big(\nicefrac{\log n}{n} \big)^{\nicefrac{1}{(2\alpha + 1)}}$ with an universal constant $\kappa$ (see \citet[Remark 3.1 and 3.2]{mosching2020monotone}).
\end{proposition}

\begin{figure}[h]
    \centering
    \vspace{-0.1in}
    \begin{tabular}{c}
        \includegraphics[width=0.9\linewidth]{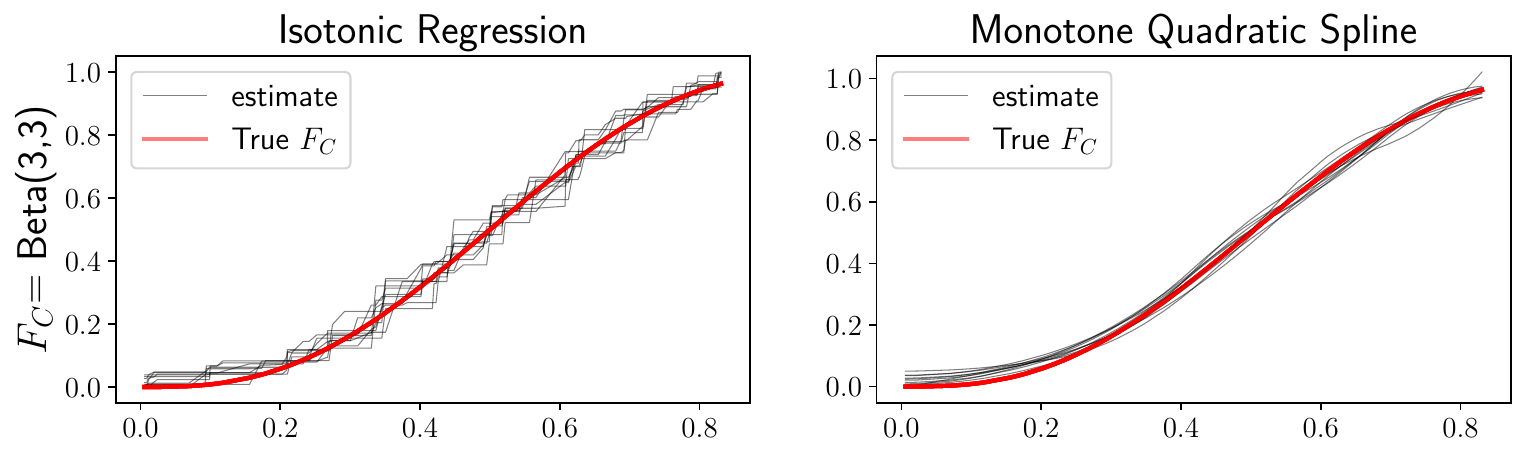}\\
        \includegraphics[width=0.9\linewidth]{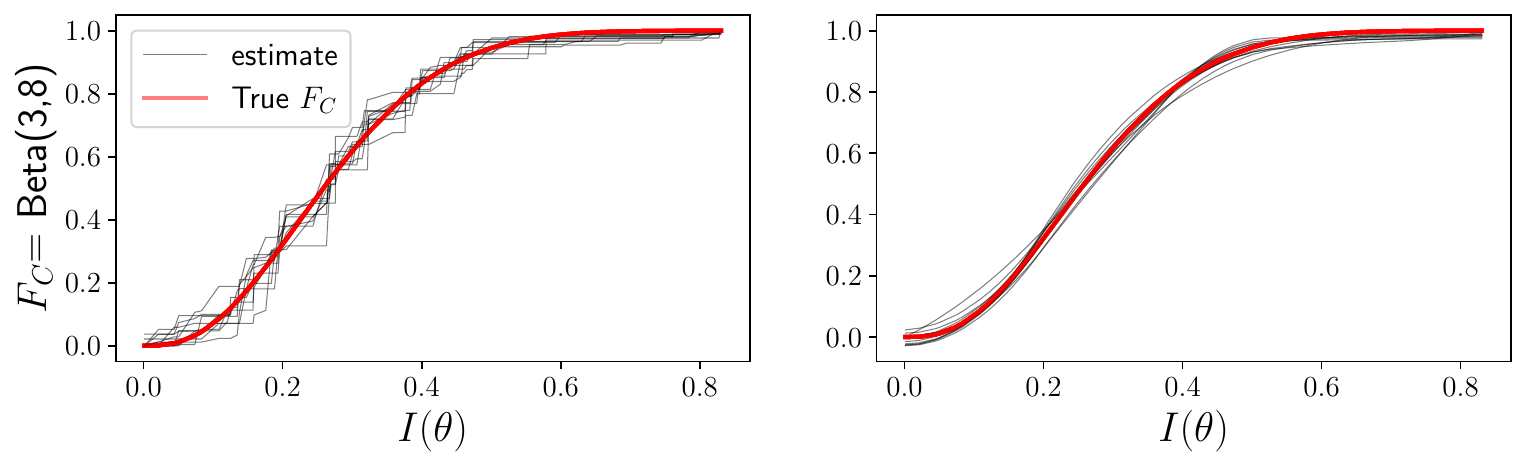}
    \end{tabular}
    \vspace{-0.15in}
    \caption{Estimation of $F_C$}
    \label{fig:CL-estimation-F}
\end{figure}

\paragraph{Estimating proportions with unknown actions:} Here we employ the second method proposed in \cref{remark:estimate-hatpi}, where the discriminating function $h_m$ is chosen simply as the classifier for action $A$ using $X$. This classifier can be fitted on a previously collected dataset, where both $A$ and $X$ are observed. Having obtained the $\widehat\pi(\theta_m)$ for thresholds $\theta_m$, we use \eqref{eq:cdf-CL-model} to estimate $\widehat F_C$. In \cref{fig:CL-estimation-F} we compare the $\widehat F_C$ with the true $F_C$, where we utilized two techniques for estimating a monotonically increasing functions: isotonic regression \citep{robertson1988order} and monotone quadratic spline \citep{meyer2008inference}. As the latter approach yields a smooth estimate, this is a better estimator, because the true $F_C$ is smooth. 

% \textbf{\textit{Estimate $\pi_m$ using features.}} For every threshold $\theta_m$, $m \in [M]$, we observe independent scores $\{X_{i,m}\}_{i\in [n_m]}$, drawn from a distribution with density $f_m(x)=\pi(\theta_m)f(x \mid 1)+(1-\pi_m)f(x \mid 0)$, and estimate $\pi_m$ using MLE, that is $\widehat{\pi}_m $ equals
% $$
% \underset{p\in [0,1]}{\argmax} \sum\nolimits_{i=1}^{n_m} \log\{pf(X_t|1)+(1-p)f(X_t|0)\}.
% $$ 
% Define $\widehat F_C$ as in \eqref{ex:eq:estimate_p} with the corresponding $\widehat{\pi}_m$ defined by the MLE. In Figure \ref{fig:CL-estimation-F} we show a simulation result for estimating the $\cD_f$ that is $F_C$ in this particular example.

While we move on to design of model $\theta_m$ that is to be deployed, in the appendix, we demonstrate how $F_C$ can be estimated within an appropriate parametric model (see \cref{sec:parametric_model_CL}), and results for estimating distribution map with multiple actions $|\cA| \ge 3$ (see \cref{sec:multivariate-estimation}).
\section{Optimal design for deploying models under binary actions}\label{sec:optimal_design}

In this section, we assume that $\cA = \{0,1\}$, the deployed models $\theta_m$ belong to a known bonded convex set $\Theta$, and without loss of generality, $\sB_0(\theta)-\sB_1(\theta) \in [0,1]$ for every $\theta \in \Theta$, so that $b_m = b(\theta_m) \triangleq \bL_0 \bB(\theta_m) = \sB_0(\theta_m)-\sB_1(\theta_m) \in [0,1]$. For a function $\theta \mapsto f(\theta)$, we denote $\|f\|_{L^{\infty}(\Theta)}= \sup_{\theta \in \Theta}|f(\theta)|$. We propose an optimal design algorithm that aims to find the distribution $d$ of the $\iid$ values $b_m \in [0,1]$  to minimize the mean integrated square error $$
\textstyle \operatorname{MISE}\{\widehat{\sF} \mid d\} \triangleq \Ex_d\int\{\widehat \sF(b) - \sF(b)\}^2 \mathrm{d} b
$$ between the true cdf $\sF \triangleq \sF_1$, ($\sF_0 = 1-\sF$) and the estimated $\widehat \sF \triangleq \widehat \sF_1$ as defined in \cref{sec:learning-cost}. Whenever we say that we draw a $b_m \in [0,1]$, we also mean dawning the corresponding $\theta_m \in \Theta$ that leads to $b_m = b(\theta_m)$. 
\begin{example}
To build an intuition, assume that for each $m \in [M]$
\[\textstyle
\widehat \pi(\theta_m) = \sF(b_m) + \epsilon(b_m), ~~ \Ex[\epsilon] = 0, \var[\epsilon] = \sigma^2(b_m)\,.
\] 
% \todo[inline]{Should $\sigma^2$ depend on $b_m$?}
Since $\mathcal{A}=\{0,1\}$, the estimator in \eqref{eq:coordinate-wise-monotone-regression} reduces to $\widehat \sF = \argmin_{F \in \cF} \sum\nolimits_m n_m (\widehat \pi (\theta_m) - F(b_m))^2$.    
\end{example}
The design density $d$ that minimizes the mean integrated squared error $\operatorname{MISE}\{\widehat{\sF} \mid d\}$ is given by 
\begin{equation}
    \label{def:optimal_design}
d^{\star}(b)=\nicefrac{\sigma(b)}{\int \sigma(u)\mathrm{d}u},
\end{equation}
where $\sigma^2(b)=\var[\widehat{\sF}(b)-\sF(b)]$ (see \citet{muller1984optimal, zhao2012sequential}) which, intuitively affirms that more design points $b_m$ are needed where the noise level for $\widehat{\sF}(b_m)-\sF(b_m)$ is high. In particular, if $\sigma(\cdot)$ is constant, then $d^{\star}(\cdot)$ is the uniform density. Although the design density takes the form of \eqref{def:optimal_design}, we're yet to learn $\sigma(\cdot)$, which we learn through a sequential method.

\subsection{Sequential method with doubling trick}\label{sec:Sequential_method-doubling-trick}

Different approaches have been developed to estimate the optimal design; most related is \textit{episode-sequential} design by \citet{zhao2012sequential}. Further developing on this design algorithm, we propose our new algorithm that uses a doubling trick to minimize the \textit{relative efficiency loss} \citep{zhao2012sequential} with respect to the optimal design $d^{\star}$ (as described in \eqref{def:optimal_design}):
\[\textstyle
\operatorname{REL}(d) \triangleq 1-\frac{\operatorname{MISE}\{\widehat{F} \mid d^{\star}\}}{\operatorname{MISE}\{\widehat{F} \mid d\}}.\] We divide the model choices $[M]$ into $K$ episodes that increase exponentially as $L_k = \tau_o 2^{k-1}$, where $\tau_o$ is the (pre-fixed) length of the first episode. With this choice of $L_k$, we have $K = \lfloor \log_2(1+\nicefrac{M}{\tau_o})\rfloor$. We state our algorithm in \ref{alg:seq-design-optial-ours} and establish an upper bound for relative efficiency in Theorem \ref{thm:sequential-design-ours}.

%  To mitigate the problem of identifying $\delta$ as explained in Remark \ref{remark:eta-non-verifiable}, we propose a new algorithm. Firstly, for a fixed design density function $d$, \citet{zhao2012sequential} define the \textit{relative efficiency loss}, compared to the optimal design density $d^{\star}$, as
% $$
% \operatorname{REL}(d) \triangleq 1-\nicefrac{\operatorname{MISE}\{\widehat{F} \mid d^{\star}\}}{\operatorname{MISE}\{\widehat{F} \mid d\}}.
% $$
% We divide $[M]$ into exponentially increasing episodes, that is $L_k = (\tau_o 2^{k-1})^{\alpha}$, for $k=1,\dots,K$, where $\tau_o$ is the length of the first episode, and $\alpha \in (0,1]$ is fixed. With this choice of $L_k$, we have $K = \log_q(1+\nicefrac{M(q-1)}{\tau_o^{\alpha}})$ where $q = 2^{\alpha}$. We state the following Algorithm \ref{alg:seq-design-optial-ours}.

\begin{algorithm}
\caption{Sequential Design Algorithm with doubling trick}\label{alg:seq-design-optial-ours}
\begin{algorithmic}[1]
\STATE Fix $\tau_o\in \bbN$, $K = \log_2(1+\nicefrac{M}{\tau_o})$, $L_k=\tau_o 2^{k-1}$. Define the episodes $\mathcal{I}_1=\{1, \dots, L_1\}$ and $\mathcal{I}_k=\{1+\sum\nolimits_{j=1}^{k-1}L_j, \ldots, \sum\nolimits_{j=1}^{k}L_j\}$, for $k=2, \dots, K$.
\STATE Let the initial density $\widehat d^{(1)}$ be the uniform in $[0,1]$.
\FOR{$k$ in $[1, K]$}
\STATE For every $m \in \mathcal{I}_k$ sample $b_m \overset{\iid}{\sim}\widehat{d}^{(k)}$ and record $\widehat \pi (\theta_m)$ (see Remark \ref{remark:estimate-hatpi}).
\STATE Based on $\{\widehat \pi (\theta_m)\}_{m \in \mathcal{I}_k}$ update $\widehat{F}^{(k)}(b)$, $\widehat \sigma^{(k)}(b)$ and the design density $\widehat{d}^{(k+1)} (b) \propto \widehat \sigma^{(k)} (b)$.
\ENDFOR
\end{algorithmic}
\end{algorithm}

% The difference of our Algorithm \ref{alg:seq-design-optial-ours} from the method proposed by \cite{zhao2012sequential} is that for every episode $k$, $\widehat d_{k+1}^*$ is independent of the previous updates $d_{1}^*,d_{2}^*,\dots,d_{k}^*$, since it depends only on $\iid$ samples from the $k$-th episode.

\begin{theorem}\label{thm:sequential-design-ours}
Assume that there exists a $K_0$ and a $\delta \in(0,1)$ such that $\|\widehat{d}^{(k)}-d^{\star}\|_{L^\infty(0,1)}=\mathcal{O}_P((\nicefrac{\log(L_k)}{L_k})^{\delta})
$ for all $k\geq K_0$ (where $\widehat{d}^{(k)}$ is defined in Algorithm \ref{alg:seq-design-optial-ours}). Then, as $M \rightarrow \infty$
\begin{align}\label{eq:rate-of-convergence-ragret}
\operatorname{REL}(\widehat{d}^{(K)}) = \mathcal{O}_P((\nicefrac{\log(M)}{M})^{\delta} ).
\end{align}
\end{theorem}

In Theorem \ref{thm:sequential-design-ours} $\operatorname{REL}(\widehat{d}^{(K)})$, has the same rate of convergence as the method in \citep[Theorem 2]{zhao2012sequential}. The advantage of Theorem \ref{thm:sequential-design-ours} is that, at every episode $k$ being $\widehat{d}^{(k)}$ a functional of $\iid$ data, then the unknown $\delta$ can be determined using classical uniform convergence results ($\eg$ \cite{mosching2020monotone, yang2019contraction}), which was not possible in \cite{zhao2012sequential}, where the data are not $\iid$ When $\sigma(\cdot)$ is a continuous function of $\sF$ (see $\eg$ Example \ref{ex:CL-sequential-design}, where $\sigma^2\propto \sF\bar \sF$) the value $\delta$ coincides with the value of $\eta$ that satisfies $\|\widehat{\sF}^{(k)}-\sF\|_{L^{\infty}(0,1)} = \mathcal{O}_P((\nicefrac{\log(L_k)}{L_k})^{\eta})$. In such cases, for non-parametric estimates such as kernel and spline (e.g. \citep[\S 3]{zhao2012sequential}), we have $\delta = \eta= \nicefrac{\beta}{2\beta + 1}$ where $\beta$ is the degree of smoothness of $\sF$. %However, in practice, our algorithm has good performances only when $M$ is large. Indeed, being the episodes' length exponentially increasing, $K$ (the number of updates) is smaller than the $K$ in the method proposed by \cite{zhao2012sequential}, for which we expect $\operatorname{REL}$ to be initially faster because it utilizes more data for estimating the distribution map.

\begin{figure}[h]
\centering
\vspace{-0.1in}
    \centering
   \includegraphics[width=0.8\linewidth]{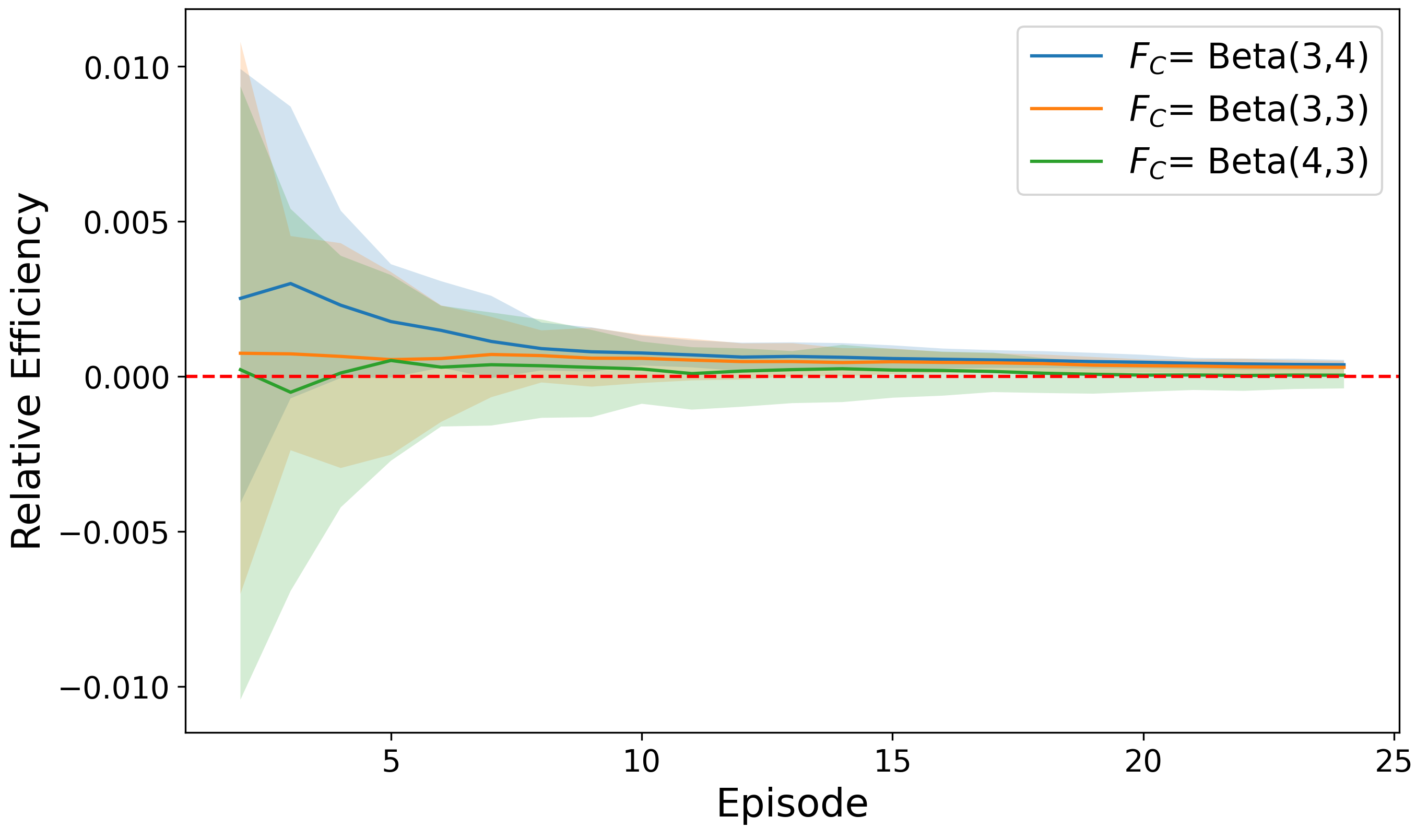}
   \vspace{-0.15in}
    \caption{Relative efficiencies and their $95\%$ confidence intervals for three different $F_C$'s.}
    \label{fig:CL-design-Beta}
    % \vspace{-0.3in}
\end{figure}

\begin{example}[CL model]\label{ex:CL-sequential-design}
Going back to \cref{ex:coate-loury}, continued in \cref{ex:coate-loury_2}, suppose that $w=1$, so that $b_m \in [0,1]$. Suppose that for every $b_m$ we observe $\widehat \pi  (\theta_m) = \nicefrac{1}{n_m} \sum\nolimits_{i=1}^{n_m} A_{m,i}$, where $n_m$ is the number of observations at point $b_m$. We have $\widehat \pi (\theta_m) = \textstyle F_C(b_m) + \epsilon(b_m)$ where $\epsilon(b_m) \triangleq \nicefrac{1}{n_m}\sum\nolimits_{i=1}^{n_m} \left(A_{m,i}-F_C(b_m)\right)$, $\mathbb{E}[\epsilon(b)]=0$ and $\sigma^2(b)=\mathbb{V}[\epsilon(b)]\propto F_C(b)\bar{F}_C(b)$. Thus, the optimal density design would be $d^{\star} (b) \propto (F_C(b)\bar{F}_C(b))^{1/2}$. Since $F_C$ is not known, we use \cref{alg:seq-design-optial-ours} to estimate it. In particular, at episode $k$, we estimate $F_C$ using isotonic regression using data  $\{b_m,\widehat{\sigma}^2(b_m) \propto \widehat{\pi} (\theta_m)(1-\widehat{\pi} (\theta_m))\}_{m\in \mathcal{I}_k}$. In Figure \ref{fig:CL-design-Beta} we see that the relative efficiency converges to zero as $K$ increases. 
% , using the method proposed by \cite{zhao2012sequential} for reasons explained at the end of Section \ref{sec:Sequential_method-doubling-trick}.

% \begin{figure}[h]
%     \centering
% \begin{tabular}{c}
%     \includegraphics[width=0.45\textwidth]{Graphs/RelEff_Beta(4,3).png}
% \end{tabular}
%     \caption{We plot the relative efficiency for 3 different $F_C$'s, a symmetric Beta(3,3), a right-skewed Beta(3,4), and a left-skewed Beta(4,3). For each distribution, we run $100$ evaluations and plot 95\% of confidence intervals. $T=500$ and $K=23$.}
%     \label{fig:CL-design-Beta}
% \end{figure}

\end{example}

\section{Regret analysis on performative risk}\label{sec:Regret}
In this section, we demonstrate how performative risk can be easily minimized by ``plugging-in'' the estimated distribution map. 
% This approach also adapts easily when the learner wishes to enforce certain constraints (\eg\ to enforce fairness after the agents adapt \citep{maity2021Does,somerstep2023Learning}), while minimizing the performative risk \citep{cotter2018Training,cotter2019Optimization}. 
As in \cref{sec:optimal_design}, assume that $\cA = \{0,1\}$ and without loos of generality $b_m = b_m(\theta_m) \in [0,1]$. In order to find the \textit{performative optimum} the learner would minimize the performative risk \eqref{eq:PR-discrete-actions}
\begin{equation}\label{eq:true-PR}
\textstyle\theta^{\star} \in \underset{\theta\in \Theta}{\argmin}\big\{\PR(\theta) \triangleq \underset{a \in \cA}{\sum} \cD_a(\theta)\Ex\big[\ell(\theta;Z)\mid A = a\big]\big\}\,,
\end{equation} where $\cD_a(\theta)$ remains unknown. In a plug-in approach, the learner estimates $\theta^\star \in \Theta$ by plugging-in an estimate  $\widehat\cD_a(\theta)$ as in \cref{alg:discrete-actions}, and then minimizes the performative prediction risk. Toward this goal of minimizing $\PR(\theta)$ what $\{\theta_m; m \in [M]\}\subset \Theta$ should be deployed? In this section, we explore this question: an optimal design choice of $\{\theta_m; m \in [M]\}$ that minimizes the total (random) regret
\begin{equation}\label{def:regret}
    \textstyle\operatorname{Reg}(M) \triangleq \sum\nolimits_{m \in [M]} \PR(\theta_{m})-\PR(\theta^{\star}),
\end{equation}
where the randomness comes from $\theta_m$. To this end, we propose a design algorithm similar to \cref{alg:seq-design-optial-ours} that relies on a doubling trick and alternates between exploration and exploitation. We divide drawing $M$ models into $K = \log_2(1+M/\tau_0)$ episodes denoted as $\{\cJ_k; k \in [K]\}$, where the lengths of the episodes are $|\mathcal{J}_k|= \tau_0 2^{k-1}$ and $\tau_0$ is the number of models drawn in the first episode (initial phase) and does not affect the rate of convergence. We then divide each episode $\mathcal{J}_k$ into two sub-phases: the \textit{exploration} phase ($\mathcal{I}_k$), where we use Algorithm \ref{alg:seq-design-optial-ours} to estimate $\cD_a(\theta)$ and then \textit{exploitation} phase ($\mathcal{I}'_k$), where we constantly deploy
\begin{align}\label{eq:estimated-PR} \textstyle
\widehat \theta^{(k)} \in \underset{\theta \in \Theta}{\argmin}\big\{\widehat{\PR}^{(k)} (\theta) \triangleq \underset{a \in \mathcal{A}}{\sum} \widehat \cD_a^{(k)}(\theta) \Ex[\ell(\theta;Z)|A=a]\big\}.
\end{align}
We summarize our procedure in \cref{alg:seq-design-doubling}.
\begin{algorithm}
\caption{Sequential Design with doubling trick}\label{alg:seq-design-doubling}
\begin{algorithmic}[1]
\STATE Let $\mathcal{J}_1,\dots,\mathcal{J}_K$ sequential episodes with $\mathcal{J}_1 = \{1,\dots, \tau_o\}$, and $\mathcal{J}_k=\mathcal{I}_k \cup \mathcal{I}'_k$ for $k=2,\dots,K$, where $\mathcal{I}_k = \{\tau_o(2^{k-1} - 1)+1, \dots, a_k\}$ is the exploration index phase, where $a_k = (\tau_o2^{k-1})^{\alpha}$ for some $\alpha \in (0,1)$, and $\mathcal{I}'_k = \{a_k+1, \dots, \tau_o(2^{k} - 1)\}$ is the exploitation index phase.
\STATE For $k=1$ let $\widehat{d}^{(1)}$ be the uniform density in $[0,1]$.
\FOR{$k=1,\dots K$}
\STATE For every $m \in \mathcal{I}_k$ sample $b_m \overset{\iid}{\sim} \widehat{d}^{(k)}$ and record $\widehat \pi(\theta_m)$ (see Remark \ref{remark:estimate-hatpi}).
\STATE Based on $\{\widehat \pi(\theta_m)\}_{m \in \mathcal{I}_k}$ update $\{\widehat \cD_a^{(k)}(\cdot);a\in \cA\}$, $\widehat \sigma^{(k)}$ and the design density $\widehat{d}^{(k+1)} \propto \widehat \sigma^{(k)}$.
\STATE For every $m \in \mathcal{I}'_k$ deploy the same $\widehat{\theta}^{(k)} = \argmax_{\theta} \widehat{\PR}^{(k)}(\theta)$, where $\widehat{\PR}^{(k)}$ is defined in \eqref{eq:estimated-PR}.
\ENDFOR
\end{algorithmic}
\end{algorithm}
Next we evaluate the goodness of \cref{alg:seq-design-doubling} in terms of the total regret \eqref{def:regret}. Notice that
there is a trade-off between the exploration and exploitation phases: a larger exploration phase
produces a more accurate $\widehat{\PR}^{(k)}(\theta)$, but does not decrease overall regret. For this reason, in \cref{alg:seq-design-doubling} we balance between these two phases by letting the length of exploration phase $a_k = |\mathcal{I}_k| = |\mathcal{J}_k|^{\alpha} = (\tau_o2^{k-1})^{\alpha}$ where $\alpha \in (0, 1)$ is chosen optimally by looking at the rate of convergence in regret. This choice of $\alpha$, along with the rate of convergence for both regret and relative efficiencies, is provided in the following theorem.

% we have defined a value $\alpha \in (0,1)$ and $a_k = |\mathcal{I}_k| = |\mathcal{J}_k|^{\alpha} = (\tau_o2^{k-1})^{\alpha}$, and we will show in Theorem \ref{thm:regret_bound} that there exists an optimal choice for such $\alpha$ that minimizes total regret, and $\alpha$ will depend on the convergence rate of $\|\cD_\theta(0)-\widehat 
% \cD_\theta(0)\|_{\infty}$. We assume that

% \begin{assumption}\label{assumption:regret}
% $\Ex[\ell(f,X)|A=a]$ is uniformly bounded in $f$ and $a$. 

% \end{assumption}

\begin{theorem}\label{thm:regret_bound}
Assume that $\Ex[\ell(\theta,X)|A=a]$ is uniformly bounded in $\theta$ and $a$.  Consider the procedure of \cref{alg:seq-design-doubling}. Suppose that there exist $\eta,\delta \in(0,1)$ and $K_0>0$ such that $\|\cD(1)-\widehat \cD^{(k)}(1)\|_{L^{\infty}(\Theta)} = \mathcal{O}_P\left((\nicefrac{\log(|\mathcal{I}_k|)}{|\mathcal{I}_k|})^{\eta}\right)$ and $\| \widehat{d}^{(k)} - d^{\star} \|_{L^{\infty}(0,1)} = \mathcal{O}_P\left((\nicefrac{\log(|\mathcal{I}_k|)}{|\mathcal{I}_k|})^{\delta}\right)$ for $k \geq K_0$. Then the total regret is minimized at $\alpha = \nicefrac{1}{1+\eta}$ for with the regret and relative efficiencies are
\[
\begin{aligned}
    \operatorname{Reg}(M) = \mathcal{O}_P\big( M^{\nicefrac{1}{1+\eta}}\log^{\eta}(M)\big), \quad\\
    \operatorname{REL}(\widehat{d}^{(K)}) = \mathcal{O}_P\big(M^{- \nicefrac{\delta}{1+\eta}}\log^{\delta}(M)\big) &\text{ as }M \rightarrow \infty.
\end{aligned}
\]
\end{theorem}

\begin{remark}
As we mention in \cref{sec:Sequential_method-doubling-trick}, when $\sigma(\cdot)$ is a continuous function of $F$ (see $\eg$ \cref{ex:CL-sequential-design}, where $\sigma^2\propto F\bar F$), then $\delta \equiv \eta$.
\end{remark}

\begin{remark}\label{remark:regret_bound}
For one-dimensional isotonic regression we have $\eta= \nicefrac{1}{3}$ (see \cref{thm:uniform_convergence}) that leads to regret $\cO_P(M^{\nicefrac{3}{4}})$. For other standard non-parametric estimates, such as kernel and spline, see \eg\ \citep[Section  3]{zhao2012sequential}, $\eta=\nicefrac{\beta}{2\beta + 1}$ where $\beta$ is the smoothness parameter for $\sF$, leading to regret $\cO(M^{\nicefrac{2\beta+1}{3\beta + 1}})$.
\end{remark}

% To compare our Theorem \ref{thm:regret_bound} with \cite[Theorem 3]{jagadeesan2022Regret}, the main difference is that we use a plug-in estimator for the distribution shift, while they require $\iid$ samples from $\cD_\theta$ in every optimization step of $\PR$.  Their regret bound in Theorem 3 is $\mathcal{O}(M^{\nicefrac{d+1}{d+2}})$, where $d$ is the zooming dimension of $\theta$ and in the worst-case it is identical to $\text{dim}(\theta)$. In contrast, our bound does not depend on $\text{dim}(\theta)$, because the distribution map itself depends on $\theta$ only through the unidimensional $b(\theta) = \bL_0 \bB (\theta)$. As a result,  we achieve $\operatorname{Reg}(M)=\widetilde{\mathcal{O}}_P(M^{\nicefrac{2\beta+1}{3\beta+1}})$ (Remark \ref{remark:regret_bound}), which is lower order than $\mathcal{O}(M^{\nicefrac{d+1}{d+2}})$ whenever $\beta > 1$ and $d\geq 2$.

Comparing our \cref{thm:regret_bound} with \citet[Theorem 3]{jagadeesan2022Regret}, the key difference lies in our use of a plug-in estimator for the $\cD(\theta)$ shift versus their need for $\iid$ samples from $\cD(\theta)$ in each optimization step of $\PR$. Their regret bound is $\mathcal{O}(M^{\nicefrac{d+1}{d+2}})$, where $d$ is the zooming dimension of $\theta$, which in the worst case is identical to $\text{dim}(\theta)$. In contrast, our bound is independent of $\text{dim}(\theta)$,because the distribution map itself depends on $\theta$ only through the unidimensional $b(\theta) = \bL_0 \bB (\theta)$. Consequently, we achieve $\operatorname{Reg}(M)=\widetilde{\mathcal{O}}_P(M^{\nicefrac{2\beta+1}{3\beta+1}})$ (\cref{remark:regret_bound}), which is of lower order than $\mathcal{O}(M^{\nicefrac{d+1}{d+2}})$ whenever $\beta \geq 1$ and $d\geq 2$.

\section{Conclusions}
\label{sec:conclusion}

We propose a microfoundation framework with random cost to learn the distribution map in reverse causal performative prediction settings. We also provide a design choice of predictive models for the learner to efficiently estimate the distribution map and minimize the performative risk. The performative prediction settings are extremely relevant in ML research, as they are routinely realized in many strategic environments. We believe that our proposal of learning the distribution map will make the problem of performative risk minimization more accessible to practitioners.

\emph{Where does our reverse causal framework fit into learning the distribution map in broader performative prediction problems?} Our focus is primarily on finite action spaces, but the reverse causal model with an arbitrary $A$ offers significant potential in performative prediction. Specifically, any $\cD(X, Y;\theta)$ can be fit into reverse causal model by setting $A = (X,Y)$, making $\cD(X, Y \mid A;\theta)$ trivially independent of $\theta$. A balance can be struck between this general $A$ and finite action spaces by choosing an action variable $A$ that captures all relevant information about $\theta$, ensuring $\cD(X, Y\mid A; \theta)$ remains independent of $\theta$. Extending our approach to arbitrary action spaces may involve using generative models, such as a generative model $h$ with a latent variable $U$ (with a known distribution), to express $\cD_A(\theta)$ via $h(U, \theta)$. For instance, in the CL model \ref{ex:coate-loury}, setting $U \sim \text{unif}(0, 1)$ and $h(U, \theta) = \indicator\{U \le \pi(\theta)\}$ allows $A$ to be modeled as $h(U, \theta)$, where $\cD_{A = 1}(\theta) = P(h(U, \theta) = 1) = \pi(\theta)$, making $A$ and $h(U, \theta)$ identically distributed. Learning the generative map $h$ could benefit from the techniques in \citet{goodfellow2014generative,arjovsky2017wasserstein}. While exploring this further could broaden learning of distribution shifts, it is outside this paper's scope.

\section{Code Availability}
\label{sec:code}

The codes are available at \url{https://github.com/dbracale/LearningDistributionPP-Code}.

% \newpage
\bibliography{main}
\bibliographystyle{plainnat}

\newpage

\onecolumn

\makeatletter
\renewcommand{\aistatstitle}[1]{
  \hsize\textwidth
  \linewidth\hsize
  \toptitlebar % This adds the top bar
  {\centering
  {\Large\bfseries #1 \par}}
  \bottomtitlebar % This adds the bottom bar
  % Adjust vertical spacing below the title if necessary
  \vskip 0.1in  % Modify this value to fine-tune spacing as per your preference
}
\makeatother

\newpage
\section*{Checklist}

\begin{enumerate}
 \item For all models and algorithms presented, check if you include:
 \begin{enumerate}
   \item A clear description of the mathematical setting, assumptions, algorithm, and/or model. [Yes]
   \item An analysis of the properties and complexity (time, space, sample size) of any algorithm. [Yes]
   \item (Optional) Anonymized source code, with specification of all dependencies, including external libraries. [Yes]
 \end{enumerate}

 \item For any theoretical claim, check if you include:
 \begin{enumerate}
   \item Statements of the full set of assumptions of all theoretical results. [Yes]
   \item Complete proofs of all theoretical results. [Yes]
   \item Clear explanations of any assumptions. [Yes]     
 \end{enumerate}

 \item For all figures and tables that present empirical results, check if you include:
 \begin{enumerate}
   \item The code, data, and instructions needed to reproduce the main experimental results (either in the supplemental material or as a URL). [Yes]
   \item All the training details (e.g., data splits, hyperparameters, how they were chosen). [Yes]
    \item A clear definition of the specific measure or statistics and error bars (e.g., with respect to the random seed after running experiments multiple times). [Yes]
    \item A description of the computing infrastructure used. (e.g., type of GPUs, internal cluster, or cloud provider). [Not Applicable]
 \end{enumerate}

 \item If you are using existing assets (e.g., code, data, models) or curating/releasing new assets, check if you include:
 \begin{enumerate}
   \item Citations of the creator If your work uses existing assets. [Not Applicable]
   \item The license information of the assets, if applicable. [Not Applicable]
   \item New assets either in the supplemental material or as a URL, if applicable. [Not Applicable]
   \item Information about consent from data providers/curators. [Not Applicable]
   \item Discussion of sensible content if applicable, e.g., personally identifiable information or offensive content. [Not Applicable]
 \end{enumerate}

 \item If you used crowdsourcing or conducted research with human subjects, check if you include:
 \begin{enumerate}
   \item The full text of instructions given to participants and screenshots. [Not Applicable]
   \item Descriptions of potential participant risks, with links to Institutional Review Board (IRB) approvals if applicable. [Not Applicable]
   \item The estimated hourly wage paid to participants and the total amount spent on participant compensation. [Not Applicable]
 \end{enumerate}
 \end{enumerate}

\newpage
\appendix


% Supplementary material: To improve readability, you must use a single-column format for the supplementary material.
\onecolumn
\aistatstitle{Learning the Distribution Map in Reverse Causal Performative
Prediction: \\
Supplementary Materials}

\section{Parametric model the distribution of the cost in the CL model}\label{sec:parametric_model_CL}

Suppose that $F_C$ is the cdf of a Gaussian $\mathcal{N}(\mu,\sigma)$ that is $g(b) = \Phi(\frac{b-\mu}{\sigma})$, where $\Phi$ is the cdf of a $\mathcal{N}(0,1)$. We have $F_C(b(\theta))= \Phi(\frac{b-\mu}{\sigma})$. If we observe models $\{\theta_m\}_{m \in [M]}$ and corresponding proportions $\{\widehat{\pi}_m\}_{m \in [M]}$ and benefits $\{b_m\}_{m \in [M]}$, to estimate $F_C$ we can solve the linear model 
$$
\phi_m \equiv \Phi^{-1}(\widehat{\pi}_m) = \frac{b_m}{\sigma}  - \frac{\mu}{\sigma} + \epsilon_m
$$
with the constraint that $\sigma>0$. The solution is
$$
\widehat{\sigma}_M = \begin{cases} \frac{\sum_{m=1}^M \Delta b_m \Delta \phi_m}{\sum_{m=1}^M \Delta b_m^2} & \text {when } \sum_{m=1}^M \Delta b_m \Delta \phi_m >0  \\ 0 & \text {otherwise }\end{cases}
$$
and $\widehat{\mu}_M = \bar{b}-\widehat{\sigma}_M\bar{\phi}$, where $\Delta b_m = b_m - \bar{b}$, $\Delta \phi_m = \phi_m - \bar{\phi}$, and $\bar{b}$ and $\bar{\phi}$ are the average of the $b_m$ and $\phi_m$, respectively. Similarly, we can choose a logistic model by replacing the standard Gaussian cdf with $\Phi(b) = 1/(1+e^{-b})$, the standard logistic cdf. The advantages of parametric models are simplicity of interpretation, computational efficiency, and convergence rate $\mathcal{O}_P(\sqrt{M})$.

\section{Simulation: multivariate cost estimation}\label{sec:multivariate-estimation}

\begin{figure}[h]
    \centering
    \begin{tabular}{cc}
        \hspace*{-0.2in}\includegraphics[width=0.35\textwidth]{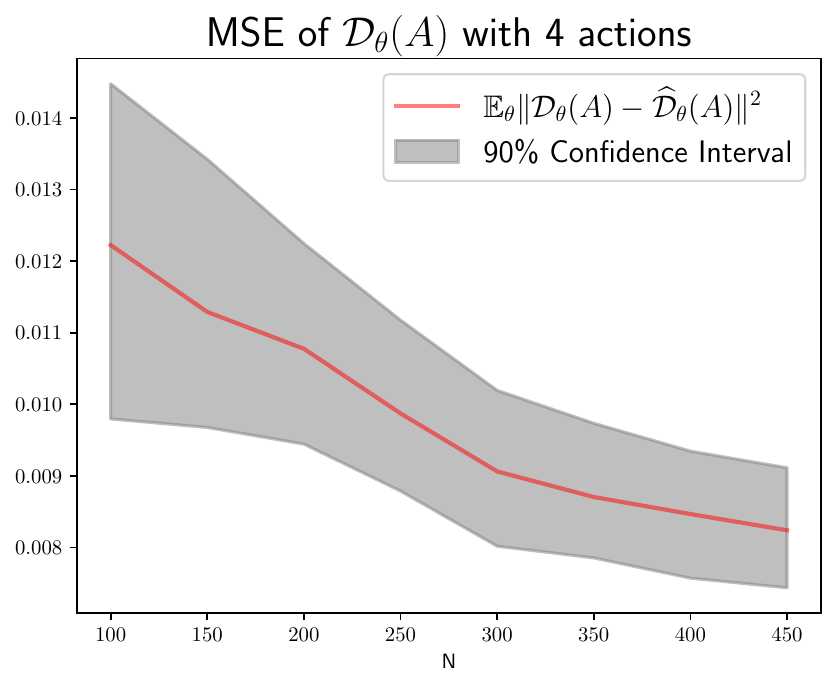} &
        \hspace*{-0.1in}\includegraphics[width=0.35\textwidth]{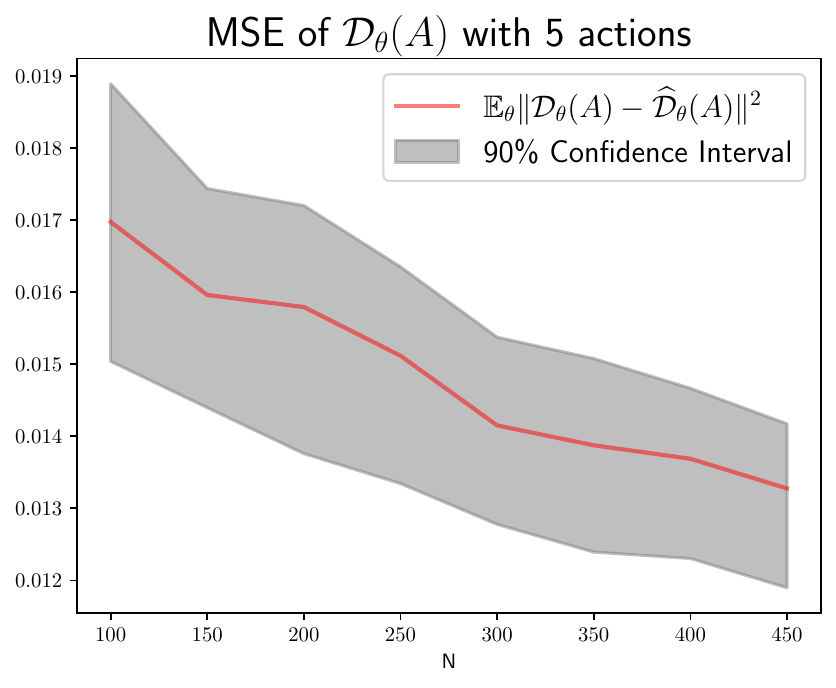}
    \end{tabular}
    \caption{Multivariate distribution map estimation.}
    \label{fig:CL-estimation-F-multivariate}
\end{figure}

In this section, we provide a simulation setup for a general number of actions $|\cA|$. We set $\bC \sim \cN_{|\cA|}(\boldsymbol{\mu},\Sigma)$, with $\boldsymbol{\mu}= \bold{0}$ and $\Sigma= \bI$. We select random benefits $\{\bB(\theta_m), m \in [M]\}$ with $M=500$. For every $a \in \cA$, we define $\widehat \pi_a(\theta_m) =  \pi_a(\theta_m)+ \epsilon_a(\theta_m)$, where $\pi_a(\theta_m) =\sF_{a}(L_a \bB(\theta_m))$, where $\sF_a$ is the cdf of $\bL_a \bC \sim \cN_{|\cA|-1}(\bL_a\boldsymbol{\mu},\bL_a \Sigma \bL_a^{\top})$ and $\iid$ $\epsilon_m(a) \sim \cN(0,\tau)$, with $\tau = 0.1$. We redefined $\widehat \pi_a(\theta_m) \leftarrow \max\{0,\min\{1,\widehat \pi_a(\theta_m)\}\}$ since they represents probabilities. For $a \in \cA \setminus \{0\}$, we get $\widehat F_a$ as multivariate isotonic regression with data $\{(\widehat \pi_a(\theta_m), \bL_a\bB(\theta_m)), m \in [M]\}$, and $\widehat \sF_{0} = 1-\sum\nolimits_{a \in \cA \setminus \{0\}} \widehat \sF_a$. Then we have $\widehat \cD_A({\theta_m}) = \{\widehat \sF_a(\bL_a\bB(\theta_m)); a \in \cA\}$ and $\cD_A({\theta_m}) = \{\sF_a(\bL_a\bB(\theta_m)); a \in \cA\}$ and we plot the quantity $\sum\nolimits_{m\in [N]}\|\cD_A({\theta_m})-\widehat \cD_A({\theta_m}) \|_2^2$, for $N \in \{50,100,150,200,\dots,500\}$. We repeat this process for $10$ times, with the same $\bC$ and $\{\bB(\theta_m), m \in [M]\}$ and with different seed to simulate $\{\pi_a(\theta_m), a \in \cA, m \in [M]\}$ and we show the result in Figure \ref{fig:CL-estimation-F-multivariate}.

\section{Regret Analysis: comparison with zero order methods}\label{sec:regret_comparison}

We compare our result stated in Theorem \ref{thm:regret_bound}, with the result provided by \cite{jagadeesan2022Regret}. They define
$$
\operatorname{Reg}_{\textit{mean}}(M)= \textstyle\sum_{m=1}^M \big[ \mathbb{E}[\PR(\theta_{m})]-\min_\theta\PR(\theta)\big]
$$
where the expectation is taken over the possible randomness in the choice of $\{\theta_m\}_{m\in[M]}$, a sequence of deployed models. The main difference between the two algorithms is that, while we estimate $\PR$ by plugging in an estimate of the distribution shift, they sample from the distribution shift to estimate $\PR$. In Theorem 4.2 they prove that 
\begin{equation*}
\operatorname{Reg}_{\textit{mean}}(M)=\mathcal{O}(M^{\frac{d+1}{d+2}}\epsilon^{\frac{d}{d+2}}+\sqrt{M})
\end{equation*}
where $d$ is the zooming dimensions (which is equal to the dimension of $\theta$ in the worst-case scenario) and $\epsilon$ is the sensitivity of the distribution map, which is a value proportional to the amount of distribution shift. They assume that:
\begin{enumerate}
    \item[$i)$] the loss $\ell$ is bounded in $[0,1]$ (it is the same as our Assumption in Theorem \ref{thm:regret_bound}).
    \item[$ii)$] $\ell$ is Lipschitz in the two variables (which we don't require);
\end{enumerate}
When $\epsilon = 0$, meaning that the performative effect vanishes, $\operatorname{Reg}_{\textit{mean}}(M)$ achieves the standard bound $\mathcal{O}(\sqrt{M})$ in an essentially dimension-independent manner. 
In our work, we assume that the distribution shift is non-trivial ($\ie$ $\epsilon>0$). Indeed, when $\epsilon \approx 0$, the agents don't move, which means that the cost $c$ is always null, and then, there is no need to estimate it. For this reason, our regret bound can be compared to \cite{jagadeesan2022Regret} when $\epsilon > 0$ and is independent of $M$, which leads to $\operatorname{Reg}_{\textit{mean}}(M)=\mathcal{O}(M^{\frac{d+1}{d+2}})$.. However, our result is independent of the dimension of $\theta$, because the distribution map only depends on $\theta$ throw the unidimensional variable $b(\theta) = \bL_0 \bB (\theta)$. As mentioned in Remark \ref{remark:regret_bound}, we can achieve $\operatorname{Reg}(M)=\widetilde{\mathcal{O}}_P(M^{\nicefrac{2\beta+1}{3\beta+1}})$, which upper-bound is lower than $\mathcal{O}(M^{\nicefrac{d+1}{d+2}})$ (excluding $\log(M)$ factors) whenever $\beta \ge \nicefrac1{(d-1)}$, \ie\ $\beta \geq 1$ and $d\geq 2$.

\section{Proofs} \label{appendix:proofs}
\begin{proof}[Proof of Lemma \ref{lemma:cost-existance}]
\label{proof:lemma-cost-existance}    
The necessary direction is trivially realized once $\cD_{A = a}(\theta)$ is expressed as 
\begin{equation} \label{eq:tech-cost-exists-3}
    \textstyle
\cD_{A = a}(\theta) = \Pr\big\{\sB_{a}( \theta)-\sB_{a'}( \theta) \ge C_a-C_{a'}, ~ \text{for all} ~ a'\neq a\big\}\,.
\end{equation} because this necessarily implies that $\cD_{A = a}(\theta)$ increases with $\sB_{a}( \theta)-\sB_{a'}( \theta)$ for any $a ' \neq a$. Establishing the sufficient direction is more non-trivial, which we establish using an induction argument. We start with the case of $|\cA| = 2$, where, without loss of generality, we let $\cA = \{0, 1\}$. Furthermore, we set $C_0 \equiv 0$. In this case, since $\cD_{A=1}(\theta)$ increases with $\sB(1, \theta) - \sB(0, \theta) $ there exists a unique random variable $C_1 $ such that 
\[\textstyle
\cD_{A=1}(\theta) = \Pr\big\{C_1 \le \sB(1, \theta) - \sB(0, \theta) \big\}\,.
\]   For such $\{C_0, C_1\}$ we have the desired microfoundation for $\cA = \{0,1 \}$.

For a more general $|\cA| = K + 1 > 2 $ let $\cA = \{0, 1, \dots, K\} $. To establish our induction step we assume that there exists a joint distribution for the $\bC_{-K} = \{C_a; a = 0, \dots, K - 1\}$ such that the microfoundation holds on $\cA - \{K\}$, \ie\ 
for any $a < K$ we have
\begin{equation} \label{eq:tech-cost-exists-2}
    \textstyle \cD_{A = a \mid A \neq K}(\theta) =  \frac{\cD_{A = a}(\theta)}{1 - \cD_{A = K}(\theta)} =  \Pr\big\{\sB_{a}( \theta)-\sB_{a'}( \theta) \ge C_a-C_{a'}, ~ \text{for all} ~ a' < K, a' \neq a\big\}\,.
\end{equation}

Now, we show that there exists a random variable $C_K$ such that 
\begin{equation}
    \label{eq:tech-cost-exists-1}
    \cD_{A = K}(\theta) = \Pr\big\{\sB_{K}( \theta)-\sB_{a'}( \theta) \ge C_K-C_{a'}, ~ \text{for all} ~ a' < K\big\}\,.
\end{equation} For the process, we further simplify the above expression in the right hand side as: 
\begin{equation} \label{eq:tech-cost-exists-4}
    \begin{aligned}
    & \textstyle \Pr\big\{\sB_{K}( \theta)-\sB_{a'}( \theta) \ge C_K-C_{a'}, ~ \text{for all} ~ a' < K\big\} \\
    & \textstyle = \Pr\big\{C_K \le C_{a'}+ \sB_{K}( \theta)-\sB_{a'}( \theta), ~ \text{for all} ~ a' < K\big\}\\
    & \textstyle = \Pr\big\{C_K \le \min_{a'<K} \big(C_{a'}+ \sB_{K}( \theta)-\sB_{a'}( \theta)\big)\big\}\\
    & \textstyle = \int \Pr\big\{C_K \le \min_{a'<K} \big(C_{a'}+ \sB_{K}( \theta)-\sB_{a'}( \theta)\big)\mid \bC_{-K}\big\} d \bbP(\bC_{-K})
\end{aligned}
\end{equation}

At this point, for any fixed value for $\bC_{-K}$ we can find a distribution for $C_K \mid \bC_{-K}$ such that
\[
\cD_{A = K}(\theta) = \Pr\big\{C_K \le \min_{a'<K} \big(C_{a'}+ \sB_{K}( \theta)-\sB_{a'}( \theta)\big)\mid \bC_{-K}\big\} 
\] This is possible because $\cD_{A = K}(\theta)$ increases with each $\sB_{K}( \theta)-\sB_{a'}( \theta)$ and thus with $\min_{a'<K} \big(C_{a'}+ \sB_{K}( \theta)-\sB_{a'}( \theta)\big)$. Denote the c.d.f. of the conditional distribution of $C_K \mid \bC_{-K}$ by $\sF_k (\cdot \mid \bC_{-K})$. We combine this conditional distribution with the joint distribution of $\bC_{-K}$ to obtain our desired joint distribution of $\bC$. 

To complete the proof, we need to show that such a choice of $\bC$ satisfies \eqref{eq:tech-cost-exists-3}. We have already established it for $a = K$, \ie\  \eqref{eq:tech-cost-exists-1} because according to \eqref{eq:tech-cost-exists-4}
\[
 \begin{aligned}
    & \textstyle \Pr\big\{\sB_{K}( \theta)-\sB_{a'}( \theta) \ge C_K-C_{a'}, ~ \text{for all} ~ a' < K\big\} \\
    & \textstyle = \int \Pr\big\{C_K \le \min_{a'<K} \big(C_{a'}+ \sB_{K}( \theta)-\sB_{a'}( \theta)\big)\mid \bC_{-K}\big\} d \bbP(\bC_{-K})\\
    & \textstyle = \int \cD_{A = K}(\theta) d \bbP(\bC_{-K}) = \cD_{A = K}(\theta).
\end{aligned}
\] To show this for $a \neq K$ we follow the identities:
\[
\begin{aligned}
  & \textstyle  \Pr\big[\sB_{a}( \theta)-\sB_{a'}( \theta) \ge C_a-C_{a'}, ~ \text{for all} ~ a'\neq a\big]\\
  & = \textstyle \Pr\big[\big \{\sB_{a}( \theta)-\sB_{a'}( \theta) \ge C_a-C_{a'}, ~ \text{for all} ~ a'\neq a, a' <K\big\} \\
  & ~~~~~~~~ \cap \big \{\sB_{a}( \theta)-\sB_{K}( \theta) \ge C_a-C_{K}\big\} \big]\\
  & =\textstyle \Pr\big[\big \{\sB_{a}( \theta)-\sB_{a'}( \theta) \ge C_a-C_{a'}, ~ \text{for all} ~ a'\neq a, a' <K\big\} \\
  & ~~~~~~~~ \cap \big \{\sB_{a'}( \theta)-\sB_{K}( \theta) > C_{a'}-C_{K} ~\text{for some} ~ a' <K\big\} \big] ~~ (\text{since dist. of $C_a$ are continuous})\\
  & = \textstyle \int \mathrm{d} \bbP(\bC_{-K}) \indicator \big \{\sB_{a}( \theta)-\sB_{a'}( \theta) \ge C_a-C_{a'}, ~ \text{for all} ~ a'\neq a, a' <K\big\}\\
  & \textstyle ~~~~~~~~~ \times \Pr\big \{\sB_{a'}( \theta)-\sB_{K}( \theta) > C_{a'}-C_{K},  ~\text{for some} ~ a' <K \mid \bC_{-K}\big\} \\
  & = \textstyle \int \mathrm{d} \bbP(\bC_{-K}) \indicator \big \{\sB_{a}( \theta)-\sB_{a'}( \theta) \ge C_a-C_{a'}, ~ \text{for all} ~ a'\neq a, a' <K\big\} \cD_{A \neq K}(\theta) ~~ (\text{by \eqref{eq:tech-cost-exists-1}}) \\
& = \textstyle \cD_{A \neq K}(\theta) \textstyle \int d \bbP(\bC_{-K}) \indicator \big \{\sB_{a}( \theta)-\sB_{a'}( \theta) \ge C_a-C_{a'}, ~ \text{for all} ~ a'\neq a, a' <K\big\}  \\ 
& = \textstyle \cD_{A \neq K}(\theta) \cD_{A = a\mid A \neq K}(\theta) ~~ (\text{by \eqref{eq:tech-cost-exists-2}}) \\
& \textstyle = \cD_{A =a}(\theta)\,.  \\ 
\end{aligned}
\] This completes the proof of existance.

\paragraph{Uniqueness:} If $\{\cD_{A}(\theta): \theta \in \Theta\}  = \Delta^{\cA}$ the uniqueness of the random variable is easily realized by noticing that the collection $\{(\cD_{A}(\theta), \{\sB(a, \theta) - \sB(0, \theta)\}_{a \neq 0}): \theta \in \Theta\}$ specifies all probability-quantile pairs.
\end{proof}

\textbf{Proof of Proposition \ref{thm:sequential-design-ours}}
\begin{proof}
Recall the following definitions: $L_k = \tau_o2^{k-1}$ is the length of the $k$-th episode for some integer $\tau_o$ that defined the length of the first episode; $K=K(M) = \log_2(1+M/\tau_o)$ is the total number of episodes. 

The assumption $\| \widehat{d}^{(k)} - d^{\star} \|_{L^{\infty}(0,1)} = \mathcal{O}_P\left(\log(L_k)^{\delta}L_k^{-\delta}\right)$ for $k \geq K_0$, implies in particular for $k =K$ that for every $\varepsilon>0$ there exists $C_1>0$ and $M_0>0$ such that
\begin{align}\label{proof:them:design:unif-conv}
\|\widehat{d}^{(K)}-d^{\star}\|_{L^{\infty}(0,1)} & \leq C_1 \log(L_K)^{\delta}L_K^{-\delta} = C_1\log^{\delta}(2^K)2^{-K\delta} = C_1K^{\delta}2^{-K\delta} = C_1\log^{\delta}_2(M)M^{-\delta},
\end{align}
with probability at least $1-\varepsilon$, for all $M\geq M_0$. From now on we consider all the statements in our proof happening in this high-probability event. We will also use the notation $a_n\asymp (\pm 1)\epsilon_n + a^{\star}$ to denote $|a_n-a^{\star}|\lesssim \epsilon_n$. For example, by assumption we have
$$
\widehat{d}^{(K)}\asymp (\pm 1)\xi + d^{\star}, \quad \text{uniformly in $(0,1)$, where} \quad \xi \triangleq \log^{\delta}_2(M)M^{-\delta}.
$$

\textbf{Step 1}.

From the proof of Theorem 2-3 of \cite{zhao2012sequential} we have that $\operatorname{REL}(\widehat{d}_{K}) = \frac{4}{3}\rho_M + o(\rho_M^2)$ as $M \rightarrow \infty$ 
$$
\rho_{M} \triangleq (U_{M}-c^2)/U_{M}
$$ 
with 
$$
U_{M} \triangleq \int \frac{\sigma^2(b)}{\widehat{d}^{(K)}(b)} \d b \quad \text{and} \quad c \triangleq \int \sigma(b) \d b,
$$
so that by \cref{def:optimal_design} $d^{\star} = \sigma/c$. Without loss of generality, we assume that $b \in (0,1)$. 

\textbf{Step 2}.

From $\widehat{d}^{(K)}\asymp (\pm 1)\xi + d^{\star}$ (uniformly in $b\in (0,1)$) we have
$$
U_M = \int \frac{\sigma^2(b)}{\widehat{d}^{(K)}}\d b \asymp  \int \frac{\sigma^2(b)}{(\pm 1)\xi + d^{\star}(b)}\d b
$$ 
Then
\begin{align*}
\frac{U_M-c^2}{c} \asymp  \int \frac{\sigma^2(b)}{c((\pm 1)\xi + d^{\star}(b))}\d b - \int \sigma (b) \d b = \int \frac{\sigma^2(b)-c\sigma(b)((\pm 1)\xi + d^{\star}(b))}{c((\pm 1)\xi + d^{\star}(b))} \d b
\end{align*}
Using that $d^{\star}(b) = \sigma(b)/c$ we have
\begin{align*}
\frac{U_M-c^2}{c} \asymp  -(\pm 1) \int \frac{\xi\sigma(b)}{\xi + d^{\star}(b)} \d b
\end{align*}
By the Taylor expansion $\tfrac{\xi \sigma}{\xi + d^{\star}} = \frac{\sigma}{d^{\star}}\xi + o(\xi^2)$ as $\xi \rightarrow 0$, since $\xi \rightarrow 0$ as $M \rightarrow \infty$, we have
$$
\frac{U_M-c^2}{c} \asymp  -(\pm 1) \xi \int \frac{\sigma (b)}{d^{\star}(b)} \d b + o(\xi^2) = - (\pm 1)\xi c + o(\xi^2), \quad M \rightarrow \infty
$$ 
Then using that $U_M = \int \frac{\sigma^2(b)}{\xi + d^{\star}(b)}\d b \rightarrow \int \frac{\sigma^2(b)}{d^{\star}(b)}\d b = c \int \sigma(b)\d b = c^2$ as $M\rightarrow \infty $ we have 
\begin{align*}
\rho_M &= \frac{U_M-c^2}{U_M} = \frac{U_M-c^2}{c} \frac{c}{U_M} \asymp  -(\pm 1) \xi + o (\xi^2)
\end{align*}
then as $M\rightarrow \infty$, $\rho_M \asymp  (\pm 1)\xi = (\pm 1) \log_2^{\delta}(M)M^{-\delta}$ that is $\rho_M = \mathcal{O}_P \left(\log_2^{\delta}(M)M^{-\delta} \right)$.
\end{proof}

\textbf{Proof of Theorem \ref{thm:regret_bound}}
\begin{proof}

From definition \eqref{def:regret}
\begin{align}\label{eq:proof_regret_regret_def_1}
\operatorname{Reg}(M) &= \sum_{k=1}^K\sum_{m \in \mathcal{J}_k} \PR(\theta_m^{(k)})- \PR(\theta^{\star}) \nonumber\\
&= \sum_{k=1}^K \left[ \sum_{m \in \mathcal{I}_k} \PR(\theta_m^{(k)})- \PR(\theta^{\star}) + \sum_{m \in \mathcal{I}_k'} \PR(\theta_m^{(k)})- \PR(\theta^{\star})\right]
\end{align}

For $m \in \mathcal{I}_k$, by the assumption of boundedness of the loss function, we have that $\PR(\theta_m^{(k)})- \PR(\theta^{\star})$ is bounded, then

\begin{align}\label{eq:proof_regret_regret_def_2}
\sum_{m \in \mathcal{I}_k} \PR(\theta_m^{(k)})- \PR(\theta^{\star}) \lesssim | \mathcal{I}_k|
\end{align}

For $m \in \mathcal{I}_k'$, where $\mathcal{I}_k'$ is the time-index of the exploitation phase, we split (for simplicity of notation we remove the dependence on $m$ in $\theta^{(k)}_m$)

\begin{align}\label{ineq:bound_V_1_and_V_2}
0 \leq \PR(\theta^{(k)}) - \PR(\theta^{\star}) &= \PR(\theta^{(k)}) - \widehat \PR^{(k)}(\theta^{(k)}) +\widehat \PR^{(k)}(\theta^{(k)}) - \PR(\theta^{\star}) \nonumber\\
&\leq \underset{=:\mathcal{V}_1}{\underbrace{|\PR(\theta^{(k)}) - \widehat \PR^{(k)}(\theta^{(k)})|}} +\underset{=:\mathcal{V}_2}{\underbrace{|\PR(\theta^{\star})-\widehat \PR^{(k)}(\theta^{(k)})|}}
\end{align}

\textbf{Bound $\mathcal{V}_1$.}

\begin{align}\label{ineq:bound_V_1}
\mathcal{V}_1 & = \left|\PR(\theta^{(k)})-\widehat\PR^{(k)}(\theta^{(k)}) \right| \nonumber\\
& = \sum_{a \in \mathcal{A}} \Big| \Ex_Z [\ell(\theta^{(k)};Z)|A=a] \cD_a(\theta^{(k)})-\Ex_Z [\ell(\theta^{(k)};Z)|A=a]\widehat \cD_{a}^{(k)}(\theta^{(k)}) \Big|\nonumber\\
&= \sum_{a \in \mathcal{A}}\Ex_Z [\ell(\theta^{(k)};Z)|A=a]\Big| \cD_{a}(\theta^{(k)})-\widehat \cD_{a}^{(k)}(\theta^{(k)})\Big|\nonumber\\
&\lesssim \sum_{a \in \mathcal{A}}\Big| \cD_{a}(\theta^{(k)})-\widehat \cD_{a}^{(k)}(\theta^{(k)})\Big|\nonumber\\
&=\Big| \cD_1(\theta^{(k)})-\widehat \cD_{1}^{(k)} (\theta^{(k)})\Big|+\Big| 1-\cD_1(\theta^{(k)})- (1-\widehat \cD_{1}^{(k)} (\theta^{(k)}))\Big|\nonumber\\
& \lesssim \Big| \cD_1(\theta^{(k)})-\widehat \cD_{1}^{(k)} (\theta^{(k)})\Big|\nonumber\\
& \lesssim \| \cD_1-\widehat \cD_{1}^{(k)}\|_{L^{\infty}(\Theta)}
\end{align}

\textbf{Bound $\mathcal{V}_2$.}

From definition \eqref{eq:true-PR} and \eqref{eq:estimated-PR} we have 
$$
\theta^{\star} \in \argmin_{\theta}\left\{\PR(\theta) = \sum_{a \in \mathcal{A}} \Ex_Z [\ell(\theta;Z)|A=a] \cD_a(\theta)\right\}
$$
and
$$
\theta^{(k)} \in \argmin_{\theta} \left\{\widehat \PR^{(k)}(\theta) = \sum_{a \in \mathcal{A}} \Ex_Z [\ell(\theta;Z)|A=a] \widehat \cD_a^{(k)}(\theta)\right\}.
$$

Then
$$
\PR(\theta^{\star})-\widehat\PR^{(k)}(\theta^{(k)}) = \PR(\theta^{\star})-\PR(\theta^{(k)})+\PR(\theta^{(k)})-\widehat\PR^{(k)}(\theta^{(k)}) \leq \PR(\theta^{(k)})-\widehat\PR^{(k)}(\theta^{(k)}),
$$
where we used that $\PR(\theta^{\star})\leq \PR(\theta^{(k)})$. Then
\begin{align}\label{ineq:PR-hatPR}
\PR(\theta^{\star})-\widehat\PR^{(k)}(\theta^{(k)})&\leq \PR(\theta^{(k)})-\widehat\PR^{(k)}(\theta^{(k)})\nonumber\\
& \leq \left|\PR(\theta^{(k)})-\widehat\PR^{(k)}(\theta^{(k)}) \right| \nonumber\\
& \leq \sum_{a \in \mathcal{A}} \Big| \Ex_Z [\ell(\theta^{(k)};Z)|A=a] \cD_a(\theta^{(k)})-\Ex_Z [\ell(\theta^{(k)};Z)|A=a]\widehat \cD_{a}^{(k)}(\theta^{(k)}) \Big|\nonumber\\
&= \sum_{a \in \mathcal{A}}\Ex_Z [\ell(\theta^{(k)};Z)|A=a]\Big| \cD_{a}(\theta^{(k)})-\widehat \cD_{a}^{(k)}(\theta^{(k)})\Big|\nonumber\\
&\lesssim \sum_{a \in \mathcal{A}}\Big| \cD_{a}(\theta^{(k)})-\widehat \cD_{a}^{(k)}(\theta^{(k)})\Big|\nonumber\\
&=\Big| \cD_1(\theta^{(k)})-\widehat \cD_{1}^{(k)} (\theta^{(k)})\Big|+\Big| 1-\cD_1(\theta^{(k)})- (1-\widehat \cD_{1}^{(k)} (\theta^{(k)}))\Big|\nonumber\\
& \lesssim \Big| \cD_1(\theta^{(k)})-\widehat \cD_{1}^{(k)} (\theta^{(k)})\Big|\nonumber\\
& \lesssim \| \cD_1-\widehat \cD_{1}^{(k)}\|_{L^{\infty}(\Theta)}
\end{align}
Similarly,
$$
-[\PR(\theta^{\star})-\widehat\PR^{(k)}(\theta^{(k)})] = -\PR(\theta^{\star})+\widehat\PR^{(k)}(\theta^{\star})-\widehat\PR^{(k)}(\theta^{\star})+\widehat\PR^{(k)}(\theta^{(k)}) \leq \widehat\PR^{(k)}(\theta^{\star})-\PR(\theta^{\star}),
$$ 
where we used that $\widehat\PR^{(k)}(\theta^{(k)})\leq \widehat\PR^{(k)}(\theta^{\star})$. Then
\begin{align}\label{ineq:-[PR-hatPR]}
-[\PR(\theta^{\star})-\widehat\PR^{(k)}(\theta^{(k)})]&\leq \PR(\theta^{\star})-\widehat\PR^{(k)}(\theta^{\star})\nonumber\\
& \leq\left|\PR(\theta^{\star})-\widehat\PR^{(k)}(\theta^{\star}) \right| \nonumber\\
& \leq \sum_{a \in \mathcal{A}} \Big| \Ex_Z [\ell(\theta^{\star};Z)|A=a] \cD_a(\theta^{\star})-\Ex_Z [\ell(\theta^{\star};Z)|A=a]\widehat \cD^{(k)}_a(\theta^{\star}) \Big|\nonumber\\
&= \sum_{a \in \mathcal{A}}\Ex_Z [\ell(\theta^{\star};Z)|A=a]\Big| \cD_a(\theta^{\star})-\widehat \cD^{(k)}_a(\theta^{\star})\Big|\nonumber\\
&\lesssim \sum_{a \in \mathcal{A}}\Big| \cD_a(\theta^{\star})-\widehat \cD^{(k)}_a(\theta^{\star})\Big|\nonumber\\
&=\Big| \cD_1({\theta^{\star}})-\widehat \cD^{(k)}_1(\theta^{\star})\Big|+\Big| 1-\cD_1(\theta^{\star})- (1-\widehat \cD_1^{(k)}(\theta^{\star}))\Big|\nonumber\\
& \lesssim \Big| \cD_1(\theta^{\star})-\widehat \cD_1^{(k)}(\theta^{\star}) \Big| \nonumber\\
& \lesssim \| \cD_1-\widehat \cD_1^{(k)} \|_{L^{\infty}(\Theta)}
\end{align}

Thus, by \eqref{ineq:PR-hatPR} and \eqref{ineq:-[PR-hatPR]}, we have $\mathcal{V}_2=|\PR(\theta^{\star})-\widehat \PR^{(k)}(\theta^{(k)})| \lesssim \| \cD_1-\widehat \cD_1^{(k)} \|_{L^{\infty}(\Theta)}=\mathcal{O}_P\left([\log(|\mathcal{I}_k|)/|\mathcal{I}_k|]^{\eta}\right )$.

\textbf{Putting together bound $\mathcal{V}_1$ and $\mathcal{V}_2$.}

Form \eqref{ineq:bound_V_1_and_V_2} and the two bounds we have that, for $m \in \mathcal{I}_k'$ (reintroducing the dependence on $m$ in $\theta^{(k)}_m$)
$$
\PR(\theta_m^{(k)})-\PR(\theta^{\star})=\mathcal{O}_P\left([\log(|\mathcal{I}_k|)/|\mathcal{I}_k|]^{\eta}\right ).
$$
This means that for every $\varepsilon>0$ there exists a $C_{1,m}^{(k)}>0$ and a $K_{0,m}^{(k)}>0$ such that 
$$
\mathbb{P}\left(\PR(\theta_m^{(k)})-\PR(\theta^{\star})\leq C_{1,m}^{(k)} [\log(|\mathcal{I}_k|)/|\mathcal{I}_k|]^{\eta}\right)\geq 1-\frac{\varepsilon}{|\mathcal{J}_k|K}, \quad \text{ for all }k \geq K_{0,m}^{(k)}.
$$
Now, in the intersection of these $|\mathcal{J}_k|$ events (which intersection has probability at least $1-\frac{\varepsilon}{K}$), using \eqref{eq:proof_regret_regret_def_2} we have
\begin{align*}
\sum_{m \in \mathcal{J}_k} \PR(\theta_m^{(k)})-\PR(\theta^{\star}) & = \sum_{m \in \mathcal{I}_k} \PR(\theta_m^{(k)})- \PR(\theta^{\star}) + \sum_{m \in \mathcal{I}_k'} \PR(\theta_m^{(k)})- \PR(\theta^{\star})\\
&\lesssim |\mathcal{I}_k|+|\mathcal{I}'_k||\mathcal{I}_k|^{-\eta}\log^{\eta}(|\mathcal{I}_k|) \\
&\lesssim 2^{\alpha k} + 2^k 2^{-\alpha k \eta} \log^{\eta}(2^{\alpha k}) \\
&\lesssim k^{\eta}(2^{\alpha k} + 2^k 2^{-\alpha k \eta}).
\end{align*}
The optimal $\alpha$ is achieved when the two exponents are equal that is $\alpha = 1-\alpha \eta $, that implies $\alpha = 1/(1+\eta)$. With this choice we get that for every $\varepsilon>0$ there exists a $C_1^{(k)}>0$ and a $K_0^{(k)}>0$ such that
\begin{align*}
\mathbb{P}\left(\sum_{m \in \mathcal{J}_k} \PR(\theta_m^{(k)})-\PR(\theta^{\star}) \leq C_1^{(k)} k^{\eta} 2^{\alpha k}\right) \geq 1-\frac{\varepsilon}{K}, \quad \text{ for all } k \geq K_0^{(k)}
\end{align*}
Now, consider the intersection of these events over the $K$ episodes (which intersection has a probability of at least $1-\varepsilon$). In this intersection, we have that, from \eqref{eq:proof_regret_regret_def_1}
$$
\operatorname{Reg}(M) \lesssim \sum_{k=1}^K k^{\eta}2^{\alpha k} \lesssim K^{\eta} \sum_{k=1}^K 2^{\alpha k} \lesssim K^{\eta}2^{\alpha K} \lesssim \log^{\eta}(M)M^{\alpha} = \log^{\eta}(M)M^{1/(1+\eta)},
$$
where we used that $K = \log_2(1+M/\tau_o)$. This proves that
$$
\operatorname{Reg}(M) = \mathcal{O}_P(\log^{\eta}(M)M^{1/(1+\eta)}).
$$
To prove the convergence of the relative efficiency, we apply Theorem \ref{thm:sequential-design-ours}.
\end{proof}

\newpage
\thispagestyle{empty}  % Optional: Removes header/footer if desired

\end{document}